\documentclass[11pt, onecolumn]{article}
\usepackage{hyperref}
\usepackage{url}
\usepackage{enumitem}
\usepackage{stmaryrd}
\usepackage[latin1]{inputenc}
\usepackage{graphicx}
\usepackage{amssymb} \usepackage{subcaption}\usepackage{caption}\usepackage{amsmath}
\usepackage{amsthm}
\usepackage{dcolumn}
\usepackage{bm}
\usepackage[usenames]{color} \usepackage[usenames, dvipsnames]{xcolor}

\usepackage{mathtools}

\usepackage[noline,boxed]{algorithm2e}

\newtheorem{theorem}{Theorem}[section]

\newtheorem{proposition}[theorem]{Proposition}
\newtheorem{lemma}[theorem]{Lemma}
\newtheorem{corollary}[theorem]{Corollary}

\newtheorem*{subproblem*}{Subproblem}
\newtheorem{definition}[theorem]{Definition} 
\newtheorem{assumption}{Assumption}

\usepackage{hyperref}

\newcommand{\mc}{\mathcal}
\newcommand{\mb}{\mathbf} \newcommand{\mbb}{\mathbb} \newcommand{\wt}{\widetilde}
  \newcommand{\tb}{\textbf}
\newcommand{\tx}{\text}

\newcommand{\Ta}{\Theta}

\newcommand{\R}{\mbb R}
\newcommand{\C}{\mbb C}
\newcommand{\Z}{\mbb Z}
\newcommand{\E}{\mbb E}

\newcommand{\ol}{\overline}
\newcommand{\wh}{\widehat}
\newcommand{\prj}{\tx{Proj}}

\newcommand{\diag}{\tx{diag}}

\newcommand{\lt}{\left}
\newcommand{\rt}{\right}

\newcommand{\eqdef}{=\vcentcolon}

\newcommand{\eps}{\epsilon}

\usepackage[margin=0.9in]{geometry}

\author{ Qingqing Huang\thanks{MIT. Email: qqh@mit.edu.}
  \and Sham M. Kakade\thanks{University of Washington. Email: sham@cs.washington.edu.}
  \and Weihao Kong\thanks{Stanford University. Email: kweihao@gmail.com.}
  \and Gregory Valiant\thanks{Stanford University. Email: valiant@stanford.edu.} }

\begin{document}

\title{Recovering Structured Probability Matrices}

\maketitle
\thispagestyle{empty}

\begin{abstract}
  We consider the problem of accurately recovering a matrix $\mbb B$ of size $M \times M$, which
  represents a probability distribution over $M^2$ outcomes, given access to an observed matrix of
  ``counts'' generated by taking independent samples from the distribution $\mbb B$.  How can
  structural properties of the underlying matrix $\mbb B$ be leveraged to yield computationally
  efficient and information theoretically optimal reconstruction algorithms?  When can accurate
  reconstruction be accomplished in the sparse data regime? This basic problem lies at the core of a
  number of questions that are currently being considered by different communities, including
  building recommendation systems and collaborative filtering in the sparse data regime, community
  detection in sparse random graphs, learning structured models such as topic models or hidden
  Markov models, and the efforts from the natural language processing community to compute ``word
  embeddings''.  Many aspects of this problem---both in terms of learning and property
  testing/estimation and on both the algorithmic and information theoretic sides---remain open.

  Our results apply to the setting where $\mbb B$ has a low rank  structure.  For this
  setting, we propose an efficient (and practically viable) algorithm that accurately
  recovers the underlying $M \times M$ matrix using {$\Ta(M)$
    samples} (where we assume the rank is a constant).  This linear sample complexity is optimal, up to constant factors, in an extremely
  strong sense: even testing basic properties of the underlying matrix (such as whether it
  has rank 1 or 2) requires $\Omega(M)$ samples.  Additionally,  we provide an even stronger lower bound showing that  distinguishing whether a sequence of observations were drawn from the uniform
  distribution over $M$ observations versus being generated by a well-conditioned Hidden Markov Model with two hidden
  states requires $\Omega(M)$ observations, while our positive results for recovering $\mbb B$ immediately imply that $\Omega(M)$ observations suffice to \emph{learn} such an HMM. This lower bound precludes sublinear-sample hypothesis
  tests for basic properties, such as identity or uniformity, as well as sublinear sample
  estimators for quantities such as the entropy rate of HMMs.

\end{abstract}

\newpage
\setcounter{page}{1}

\section{Introduction}
\label{sec:intro}
Consider an unknown $M \times M$ matrix of probabilities $\mbb B$, satisfying $\sum_{i,j} \mbb
B_{i,j} = 1$.  Suppose one is given $N$ independently drawn $(i,j)$-pairs, sampled according to the
distribution defined by $\mbb B$.  How many draws are necessary to accurately recover $\mbb B$?
What can one infer about the underlying matrix based on these samples?
How can one accurately test whether the underlying matrix possesses certain properties of interest?
How do structural assumptions on $\mbb B$ --- for example, the assumption that $\mbb B$ has low rank
--- affect the information theoretic or computational complexity of these questions?  For the
majority of these tasks, we currently lack both a basic understanding of the computational and
information theoretic lay of the land, as well as algorithms that seem capable of achieving the
information theoretic or computational limits.

This general question of making accurate inferences about a matrix of probabilities, given a matrix
of observed ``counts'' of discrete outcomes, lies at the core of a number of problems that disparate
communities have been tackling independently.
On the theoretical side, these problems include both work on community detection in stochastic block
models (where the goal is to infer the community memberships from an adjacency matrix of a graph
that has been drawn according to an underlying matrix of probabilities expressing the community
structure) as well as the line of work on recovering topic models, hidden Markov models (HMMs), and
richer structured probabilistic models (where the model parameters can often be recovered using
observed count data).
On the practical side, these problems include work on computing low-rank approximations to sparsely
sampled data, which arise in collaborative filtering and recommendation systems, as well as the
recent work from the natural language processing community on understanding matrices of word
co-occurrence counts for the purpose of constructing good ``word embeddings''.  Additionally, work
on latent semantic analysis and non-negative matrix factorization can also be recast in this
setting.

In this work, we focus on this estimation problem where the
probability matrix $\mbb B$ possesses a particular low rank structure.  While this estimation
problem is rather specific, it generalizes the basic community detection problem and the problem of learning various common models encountered in natural language processing such as \emph{probabilistic latent semantic analysis}~\cite{hofmann1999probabilistic}. Additionally, this problem encompasses the main technical challenge behind learning HMMs and topic models, in the sense that after $\mbb B$ is accurately recovered, these learning problems have a number of parameters that is a function only of the  number of topics/hidden states (which bounds the rank of $\mbb B$ and is, in practical applications, at most a few hundred) as opposed to the the dictionary/alphabet size, $M$, which, in natural language settings is typically tens of thousands.
Furthermore, this low rank case also provides a means to study how the relationships between property testing and estimation
problems differ  between this structured setting and the basic rank 1 setting that is
equivalent to simply drawing i.i.d samples from a distribution supported on $M$
elements.

We focus on the estimation of a low rank probability matrix $\mbb B$ in the sparse data regime, near
the information theoretic limit. In many practical scenarios involving sample counts, we seek
algorithms capable of extracting the underlying structure in the sparsely sampled regime.  To give
two motivating examples, consider forming the matrix of word co-occurrences---the matrix whose rows
and columns are indexed by the set of words, and whose $(i,j)$-th element consists of the number of
times the $i$-th word follows the $j$-th word in a large corpus of text.  In this context, the underlying probability matrix, $\mbb B$, represents the distribution of bi-grams encountered in written english.
In the context of recommendation system, one could consider a low rank matrix model, where the
rows are indexed by customers, and the columns are indexed by products, with the $(i,j)$-th entry
corresponding to the number of times the $i$-th customer has purchased the $j$-th product.  Here, the underlying probability matrix, $\mbb B$, models the distribution from which each customer/product purchase is drawn.

In both settings, the structure of the probability matrix underlying these observed counts contains
insights into the two domains, and in both domains we only have relatively sparse data. This is
inherent in many other natural scenarios  involving heavy-tailed distributions (including genomic settings), where despite having massive datasets, a significant fraction of the domain is observed only a single time.

Similar estimation questions have been actively studied in the community detection literature, where
the objective is to accurately recover the communities in the regime where the average degree
(e.g. the row sums of the adjacency matrix) are constant.
In contrast, the recent line of works for recovering highly structured models (such as topic models,
HMMs, etc.) are only applicable to the \emph{over-sampled} regime where the amount of data is well
beyond the information theoretic limits. In these cases, achieving the information theoretic limits
remains a widely open question.
This work begins to bridge the divide between these recent algorithmic advances in both communities.
We hope that the low rank probability matrix setting considered here serves as a jumping-off point
for the more general questions of developing information theoretically optimal algorithms for
estimating structured matrices and tensors in general, or recovering low-rank approximations to
arbitrary probability matrices, in the sparse data regime.
While the general settings are more challenging, we believe that some of our algorithmic techniques
can be fruitfully extended.

%

In addition to developing algorithmic tools which we hope are applicable to a wider class of
problems, a second motivation for considering this particular low rank case is that, with respect to
distribution learning and property testing, the entire lay-of-the-land seems to change completely
when the probability matrix $\mbb B$ has rank larger than 1.
In the rank 1 setting --- where a sample consists of 2 \emph{independent} draws from a distribution
supported on $\{1,\ldots, M\}$ --- the distribution can be learned using $\Theta(M)$ draws.
Nevertheless, many properties of interest can be tested or estimated using a sample size that is
\emph{sublinear} in $M$\footnote{Distinguishing whether a distribution is uniform versus far from
  uniform can be accomplished using only $O(\sqrt{M})$ draws, testing whether two sets of samples
  were drawn from similar distributions can be done with $O(M^{2/3})$ draws, estimating the entropy
  of the distribution to within an additive $\eps$ can be done with $O(\frac{M}{\eps \log M})$
  draws, etc.}.  However, even just in the case where the probability matrix is of rank 2, although
the underlying matrix $\mbb B$ can be represented with $O(M)$ parameters (and, as we show, it can
also be accurately and efficiently recovered with $O(M)$ sample counts), sublinear sample property
testing and estimation is generally impossible.
This result begs a more general question: \emph{what conditions must be true of a structured
  statistical setting in order for property testing to be easier than learning?}

\subsection{Problem Formulation}

We consider the following problem setup and notation:

\begin{itemize}
\item  A vocabulary consisting of $M$ ``words'', denoted by $\mc M=\{1,\dots,M\}$.
\item A low rank probability matrix $\mbb B$, of size $M \times M$, with the following structure:
$  \mbb B = \mbb P \mbb W \mbb P^\top,$ where $\mbb P$ is an $M \times r$ non-negative matrix with column sums 1,  and $\mbb W$ is p.s.d.  with $\sum_{i,j}\mbb W_{i,j} = 1.$  
\item A set of $N$ independent $(i,j)$ pairs drawn according to  $\mbb B$, with the probability of drawing $(i,j)$ given by $\mbb B_{i,j}.$  
\item An $M \times M$ matrix of ``counts'', $C$, summarizing the frequencies of each $(i,j)$ pair in the $N$ draws.
\end{itemize}

Throughout, we will make frequent use of the Poissonization technique whereby we  assume that the number of draws follows a Poisson distribution of expectation $N$.  This renders $C_{i,j}$ independent of the other entries of the count matrix, simplifying analysis.  Additionally, for both upper and lower bounds, with all but inverse exponential probability the $o(N)$ discrepancy between $N$ and $Poi(N)$ contributes only to lower order terms.

\paragraph{Notation}

Throughout the paper, we use the following standard shorthand notations.
Denote $[n]\triangleq \{1,\dots, n\}$.
$\mc I$ denotes a subset of indices in $\mc M$. For a $M$-dimensional vector $x$, we
use vector $x_{\mc I}$ to denote the elements of $x$ restricted to the indices in $\mc
I$; for two index sets $\mc I$, $\mc J$, and a $M\times M$ dimensional matrix $X$, we use
$X_{\mc I\times \mc J}$ to denote the submatrix of $X$ with rows restricting to indices in
$\mc I$ and columns restricting to indices in $\mc J$.

We use $\tx{Poi}(\lambda)$ to denote a Poisson distribution with expectation $\lambda$; we use
$\tx{Ber}(p)$ to denote a Bernoulli random variable with success probability
$p \in [0,1]$; and for a probability vector $x \in [0,1]^M$ satisfying $\sum_i x_i =1$ and an integer $t$, we use $\tx{Mul}( x ; t)$ to denote the multinomial distribution over
$M$ outcomes corresponding to $t$ draws from $[M]$ according to the distribution specified by the vector $x$.

%

\subsection{Main Results}

Our main result is the accurate recovery of a rank $R$ matrix of the form described above in the linear data regime $N = O(M)$:
\begin{theorem}[Upper bound for rank $R$, constant accuracy]
  \label{thm:rank-R1}
  Suppose we have access to $N$ i.i.d.  samples generated according to the a probability matrix
  $\mbb B = \mbb P \mbb W\mbb P^{T}$ with $\mbb P$ an $M \times R$ nonnegative matrix with column sum 1, $\mbb W$ an $R \times R$ p.s.d. matrix with entries summing to 1 and row sums
  bounded by $\sum_{j}\mbb W_{i,j} \ge w_{min}$.
  For any constants $\eps > 0, \delta > 0$ and $N = \Ta(\frac{M R^2}{w_{min}^2\epsilon^5} \log(1/\delta))$, there is an algorithm with $poly(M,\log(1/\delta))$ runtime that returns a rank $R$ matrix $\wh {\mbb B}$ such that with probability at least $1-\delta$:
 $$
    \|\wh {\mbb B}-\mbb B\|_{\ell_1} \le \epsilon.
 $$
\end{theorem}


We emphasize that our recovery is in terms of $\ell_1$ distance, namely the total variation distance between the true distribution and the recovered distribution.   In settings where there is a significant range in the row (or column) sums of $\mbb B$, a spectral error bound might not be meaningful.

Much of the the difficulty in the algorithm is overcoming the fact that the row/column sums of $\mbb B$ might be very non-uniform.  Nevertheless, our result can be compared to the community detection setting with $R$ communities (for which the row/column sums are completely uniform), for which accurate recovery can be efficiently achieved given $N = \Ta(MR^2)$ samples~\cite{chin2015stochastic}.  In our more general setting, we incur an extra factor of $w_{min}^{-1}$, whose removal might be possible with a more careful analysis of our approach.

\subsubsection{Topic Models and Hidden Markov Models}

One of the motivations for considering low rank structure of a probability matrix
$\mbb B$ is that this structure captures the structure of the matrix of expected bigrams
generated by topic models~\cite{papadimitriou1998latent,hofmann1999probabilistic} and HMMs, as described below.


\begin{definition}
  An R-\emph{topic model} over a vocabulary of size $M$ is defined by a set of $R$ distributions, $p^{(1)},\ldots,p^{(R)}$
supported over $M$ words, and a set of $R$ corresponding topic \emph{mixing weights} $w_1,\ldots,w_R$ with $\sum_i w_i = 1$.  The process of drawing a bigram $(i,j)$ consists of first randomly picking a topic $i\in [R]$ according to the distribution defined by the mixing weights, and then drawing two independent words from the distribution $p^{(i)}$ corresponding to the selected topic, $i$.
  Thus the probability of drawing a bigram $(i,j)$ is $\sum_{k=1}^R w_R p^{(k)}(i)p^{(k)}(j)$, and the underlying distribution $\mbb B$ over $(i,j)$ pairs can be expressed as $\mbb B = \mbb P \mbb W \mbb P^\top$ with $\mbb P = [p^{(1)},\ldots,p^{(R)}]$, and $\mbb W=diag(w_1,\ldots,w_R).$
\end{definition}

In the case of topic models, the decomposition of the matrix of bigram probabilities $\mbb B = \mbb P \mbb W \mbb P^\top$ has the desired form required by our Theorem~\ref{thm:rank-R1}, with $\mbb W$ nonnegative and p.s.d., and hence the theorem guarantees an accurate recovery of $\mbb B$, even in the sparse data regime.  The recovery of the mixing weights $\{w_i\}$ and topic distributions $\{p^{(i)}\}$ from $\mbb B$ requires an additional step, which will amount to solving a system of quadratic equations.  Crucially, however, given the rank $R$ matrix $\mbb B$, the remaining problem becomes a problem only involving  $R^2$ parameters---representing a linear combination of the $R$ factors of $\mbb B$ for each $p^{(i)}$---rather than recovering $MR$ parameters.

\begin{definition}
  \label{def:hmm}
  A \emph{Hidden Markov model}  with $R$ hidden states and observations over an alphabet of size $M$ is defined by an $ R \times R$ transition matrix $T$, and $R$ observation distributions $p^{(1)},\ldots,p^{(R)}.$     A sequence of observations is sampled as follows: select an initial state (e.g. according to the stationary distribution of the chain) then evolve the Markov chain according to the transition matrix $T$, drawing an observation from the $i$th distribution $p^{(i)}$ at each timestep in which the underlying chain is in state $i$th.
  
  Assuming the Markov chain has stationary distribution $\pi_1,\ldots,\pi_R$, the probability of seeing a bigram $(i,j)$ with symbol $i$ observed at the $k$th timestep and symbol $j$ observed at the $k+1$st timestep,  tends towards the following (i.e. assuming the chain is close to mixing by timestep $k$) rank $R$ probability matrix $\mbb B =  \mbb P \mbb W \mbb P^\top$, with $\mbb P = [p^{(1)},\ldots,p^{(R)}]$ and $\mbb W = diag(\pi_1,\ldots,\pi_n) T$.
\end{definition}

For HMMs, the low rank matrix of bigrams, $\mbb B= \mbb P \mbb W \mbb P^\top$, does \emph{not} necessarily have the required form---specifically the mixing matrix $\mbb W$ may not be p.s.d.---and it is unclear whether our approach can successfully recover such matrices.   Nevertheless, with slightly more careful analysis, at least in certain cases the techniques yield tight results.  For example, in the setting of an HMM with two hidden states, over an alphabet of size $M$, we can easily show that our techniques obtain an accurate reconstruction of the corresponding probability matrix $\mbb B$, and then leverage that reconstruction together with a constant amount of tri-gram information to accurately learn the HMM:

\begin{proposition}\label{prop:hmm}
(Learning 2-state HMMs)   Consider a sequence of observations given by a Hidden Markov Model with two hidden states and symmetric transition matrix with entries bounded away from 0.  Assuming a constant $\ell_1$ distance between the  distributions of observations corresponding to the two states, there exists an algorithm which,
  given a sampled chain of length $N=\Omega(M/\eps^2)$, runs in time
  $\textrm{poly}(M)$ and returns estimates of the transition matrix and two observation distributions that are accurate in $\ell_1$ distance, with probability at least $2/3$.
\end{proposition}  

This probability of failure can be trivially boosted to $1-\delta$ at the expense of an extra factor of $\log(1/\delta)$ observations.

\subsubsection{Testing vs. Learning}


Theorem~\ref{thm:rank-R1} and Proposition~\ref{prop:hmm} are tight in an extremely strong sense: for both the
topic model and HMM settings, it is information theoretically impossible to perform even
the most basic property tests using fewer than $\Theta(M)$ samples.
For topic models, the community detection lower
bounds~\cite{mossel2014consistency}\cite{krzakala2013spectral}\cite{zhangminimax} imply
that $\Theta(M)$ bigrams are necessary to even distinguish between the case that the
underlying model is the uniform distribution over bigrams versus the case of a
$R$-topic model in which each topic  has a unique
subsets of $M/R$ words with a constant fraction higher probability than the remaining words.
More surprisingly, for $k$-state HMMs with $k \ge 2$, even if we permit an estimator to have more information than merely
bigram counts, namely access to the \emph{full sequence} of observations, we prove the following
linear lower bound.
\begin{theorem}
  \label{thm:testing-hmm}
There exists a constant $c>0$ such that for sufficiently large $M$, given a sequence of observations from a HMM with two states and
emission distributions $p,q$ supported on $M$ elements, even if the
underlying Markov process is symmetric, with transition probability
$1/4$, it is information theoretically impossible to distinguish the
case that the two emission distributions, $p=q=\textrm{Unif}[M]$ from the case
that $||p-q||_1 =1$ with probability greater than $2/3$ using a sequence of fewer than $cM$
observations.
\end{theorem}

This immediately implies the following corollary for estimating the \emph{entropy rate} of an HMM.

\begin{corollary}
There exists an absolute constant $c>0$ such that given a sequence of observations from a HMM with two hidden states and emission distributions supported on $M$ elements, a sequence of $cM$ observations is information theoretically necessary to estimate the entropy rate to within an additive $0.5$ with probability of success greater than $2/3$.
\end{corollary}


These strong lower bounds for property testing and estimation are striking for several
reasons.  First, the core of our learning algorithm for 2-state HMMs (Proposition~\ref{prop:hmm}) is a matrix reconstruction step that
uses only the set of bigram counts.  Conceivably, it might be helpful to consider longer sequences of observations --- even for HMMs that mix in constant time,
there are detectable correlations between observations separated by $O(\log M)$ steps.
Regardless, our lower bound shows that actually no additional information from such longer
$k$-grams can be leveraged to yield sublinear sample property testing or estimation.

A second notable point is the apparent brittleness of sublinear property testing and estimation as
we deviate from the standard (unstructured) i.i.d sampling setting.   Indeed for nearly all distributional property estimation or testing tasks, including testing uniformity and estimating the entropy, sublinear-sample testing and estimation is possible in the i.i.d. sampling setting (e.g.~\cite{goldreich_ron,vv_nips,gpvaliant_power}).  In contrast to the i.i.d. setting in which estimation and testing require asymptotically fewer samples than \emph{learning}, as the above results illustrate, even in the setting of an HMM with just two hidden states, learning and testing require comparable numbers of observations.

\subsection{Related Work}

As mentioned earlier, the general problem of reconstructing an underlying matrix of
probabilities given access to a count matrix drawn according to the corresponding
distribution, lies at the core of questions that are being actively pursued by several
different communities.  We briefly describe these questions, and their relation to the
present work.

\noindent \textbf{Community Detection.}  With the increasing prevalence of large scale
social networks, there has been a flurry of activity from the algorithms and probability
communities to both model structured random graphs, and understand how (and when it is
possible) to examine a graph and infer the underlying structures that might have given
rise to the observed graph.  One of the most well studied community models is the
\emph{stochastic block model}~\cite{holland1983stochastic}.  In its most basic form, this
model is parameterized by a number of individuals, $M$, and two probabilities,
$\alpha,\beta$.  The model posits that the $M$ individuals are divided into two
equal-sized ``communities'', and such a partition defines the following random graph
model: for each pair of individuals in the same community, the edge between them is
present with probability $\alpha$ (independently of all other edges); for a pair of
individuals in different communities, the edge between them is present with probability
$\beta < \alpha$.  Phrased in the notation of our setting, the adjacency matrix of the
graph is generated by including each potential edge $(i,j)$ independently, with
probability $\mbb{B}_{i,j}$, with $\mbb{B}_{i,j} =\alpha$ or $\beta$ according to whether
$i$ and $j$ are in the same community.  Note that $\mbb B$ has rank 2 and is expressible as $\mbb{B} = PWP^\top$ where $P = [p,q]$ for
vectors $p = \frac{2}{M} I_1$ and $q=\frac{2}{M} I_2$ where $I_1$ is the indicator vector
for membership in the first community, and $I_2$ is defined analogously, and $W$ is the $2
\times 2$ matrix with $\alpha \frac{M^2}{4}$ on the diagonal and $\beta \frac{M^2}{4}$ on
the off-diagonal.  

What values of $\alpha,\beta,$ and $M$ enable the community
affiliations of all individuals to be accurately recovered with high
probability?  What values of $\alpha,\beta,$ and $M$ allow for the
graph to be distinguished from an Erdos-Renyi random graph (that has
no community structure)?   The crucial regime is where $\alpha,\beta =
O(\frac{1}{M}),$ and hence each person has a constant, or logarithmic
expected degree. The naive spectral approaches will fail in this
regime, as there will likely be at least one node with degree $\approx
\log M / \log \log M$, which will ruin the top eigenvector.
Nevertheless, in a sequence of works sparked by the paper of Friedman,
and Szemeredi~\cite{friedman1989second}, the following punchline has
emerged: the naive spectral approach will work, even in the constant
expected degree setting, provided one first either removes, or at
least diminishes the weight of these high-degree problem vertices
(e.g.~\cite{feige2005spectral,keshavan2009matrix,mossel2012stochastic,krzakala2013spectral,le2015sparse}).
For both the \emph{exact} recovery problem and the
detection problem, the exact tradeoffs between $\alpha,\beta,$ and $M$
were recently established, down to subconstant
factors~\cite{mossel2014consistency,abbe2014exact,massoulie2014community}.
More recently, there has been further research investigating more
complex stochastic block models, consisting of three or more
components, components of unequal sizes, etc. (see
e.g.~\cite{chin2015stochastic,abbe2015community,abbe2015multi}).

The community detection setting generates an adjacency matrix with entries in $\{0,1\}$, choosing entry $C_{i,j} \leftarrow Bernoulli(\mbb B_{i,j})$, as opposed to our setting where $C_{i,j}$ is drawn from the corresponding Poisson distribution.  Nevertheless, the two models are extremely similar in the sparse regime considered in the community detection literature, since, when $\mbb B_{i,j} = O(1/M),$ the corresponding Poisson and Bernoulli distributions have total variation distance $O(1/M^2).$

\medskip
\noindent \textbf{Word Embeddings.}  On the more applied side, some of the most impactful advances
in natural language processing over the past five years has been work on ``word
embeddings''~\cite{mikolov2013efficient,NIPS2014_5477,stratos_embed, arora_embedding}.  The main
idea is to map every word $w$ to a vector $v_w \in \mathbb{R}^d$ (typically $d \approx 500$) in such
a way that the geometry of the vectors captures the semantics of the word.\footnote{The goal of word
  embeddings is not just to cluster similar words, but to have semantic notions encoded in the
  geometry of the points: the example usually given is that the direction representing the
  difference between the vectors corresponding to ``king'' and ``queen'' should be similar to the
  difference between the vectors corresponding to ``man'' and ``woman'', or ``uncle'' and ``aunt'',
  etc.}  One of the main constructions for such embeddings is to form the $M\times M$ matrix whose
rows/columns are indexed by words, with $(i,j)$-th entry corresponding to the total number of times
the $i$-th and $j$-th word occur next to (or near) each other in a large corpus of text
(e.g. wikipedia).  The word embedding is then computed as the rows of the singular vectors
corresponding to the top rank $d$ approximation to this empirical count matrix.\footnote{A number of
  pre-processing steps have been considered, including taking the element-wise square roots of the
  entries, or logarithms of the entries, prior to computing the SVD.}  These embeddings have proved
to be extremely effective, particularly when used as a way to map text to features that can then be
trained in downstream applications.  Despite their successes, current embeddings seem to suffer from
sampling noise in the count matrix (where many transformations of the count data are employed,
e.g. see \cite{stratos_counts})---this is especially noticeable in the relatively poor quality of
the embeddings for relatively rare words.  The theoretical
work~\cite{DBLP:journals/corr/AroraLLMR15} sheds some light on why current approaches are so
successful, yet the following question largely remains: Is there a more accurate way to recover the
best rank-$d$ approximation of the underlying matrix than simply computing the best rank-$d$
approximation for the (noisy) matrix of empirical counts?

\medskip
\noindent \textbf{Efficient Algorithms for Latent Variable Models.}
There is a growing body of work from the algorithmic side (as opposed to information theoretic) on how to
recover the structure underlying various structured statistical
settings. This body of work includes work on learning HMMs
\cite{hsu2012spectral,MR06,Chang96}, recovering low-rank structure
\cite{arora2012learning,arora2012computing,bhaskara2014smoothed}, and
learning or clustering various structured distributions such as
Gaussian mixture models
\cite{dasgupta1999learning,vempala2004spectral,moitra2010settling,belkin2010polynomial,hsu2013learning,kalai2010efficiently,GeHK15}. A number of these methods essentially can be
phrased as solving an inverse moments problem, and the work
in~\cite{anandkumar2012tensor} provides a unifying viewpoint for
computationally efficient estimation for many of these models under a
tensor decomposition perspective. In general, this body of work has
focused on the computational issues and has considered these
questions in the regime in which the amount of data is plentiful---well above the information theoretic limits.

On the practical side, the natural language processing community has considered a variety of generative and probabilistic models that fall into the framework we consider.  These include  work on \emph{probabilistic latent semantic analysis} (see e.g.~\cite{hofmann1999probabilistic,ding2006nonnegative}), including the popular \emph{latent Dirichlet allocation} topic model~\cite{blei2003latent}.  Much of the algorithmic work on recovering these models is either of a heuristic nature (such as the EM framework), or focuses on computational efficiency in the regime in which data is plentiful (e.g.~\cite{spectralLDA}.

\medskip
\noindent \textbf{Sublinear Sample Testing and Estimation.}
In contrast to the work described in the previous section on efforts to devise computationally efficient algorithms for tackling complex structural settings in the ``over--sampled'' regime, there is also significant work establishing
information theoretically optimal algorithms and (matching) lower bounds for
estimation and distributional hypothesis testing in the most basic
setting of independent samples drawn from (unstructured)
distributions.  This work includes algorithms for estimating basic statistical properties such as entropy
~\cite{paninsky_entropy,guha_entropy,gpvaliant_clt,vv_nips}, support
size~\cite{distinct_element,gpvaliant_clt}, distance between
distributions~ \cite{gpvaliant_clt,vv_nips,gpvaliant_power}, and various hypothesis tests, such as whether two distributions are very similar, versus
significantly different~
\cite{goldreich_ron,batu,chan_valiant,instance_optimal,bhas_val}, etc. While many
of these results are optimal in a worst-case (``minimax'') sense,
there has also been recent
progress on instance optimal (or ``competitive'') estimation and
testing, e.g.
\cite{orlitsky,orlitsky_classification,instance_optimal}, with
stronger information theoretic optimality guarantees. There has also
been a long line of work beginning with~\cite{birge,Batu2004} on these tasks in ``simply structured''
settings, e.g. where the domain of the distribution
has a total ordering or where the distribution is monotonic or unimodal.

\section{Recovery Algorithm}
\vspace{-.3cm} To motivate our algorithms, it will be helpful to first consider the more naive approaches.  Recall that we are given $N$ samples drawn according to the probability matrix $\mathbb B$, with $C$ denoting the matrix of empirical counts.  By the Poisson assumption on sample size, we have that $C_{i,j}\sim \tx{Poi}(N\mbb B_{i,j})$.  Perhaps the most naive hope is to consider the rank $R$ truncated SVD of the empirical matrix $\frac{1}{N}C$, which concentrates to $\mbb B$ in
Frobenius norm at $\frac{1}{\sqrt{N}}$ rate. Unfortunately, in order to achieve constant $\ell_1$ error, this approach would require a sample complexity as large as {$\Ta(M^2)$}. Intuitively, this is because the rows and columns of $C$ corresponding to words with
larger marginal probabilities have higher row and column sums in expectation, as well as higher variances that undermine the spectral concentration of the matrix as a whole.

The above observation leads to the idea of pre-scaling the matrix so that every word (i.e. row/column) 
roughly has equal variance. Indeed, with the pre-scaling modification of the truncated SVD, one can likely improve
the sample complexity of this approach to {$\Ta(M\log M)$}.   To further reduce the sample complexity, it is worth considering what prevents the truncated SVD from achieving accurate recovery in the $N=\Ta(M)$ regime. Suppose the word marginals are roughly uniform, namely all in the order of {$O({1\over M})$}, the linear sample regime roughly corresponds to the stochastic block model setup where the
expected row sums are all of order $d={N\over M}=\Omega(1)$. It is well-known that in this sparse regime, the adjacency matrix (in the graph setting), or the empirical count matrix
$C$ in our problem, does not concentrate to the expectation matrix in the spectral sense. Due to heavy rows/columns of sum $\Omega({\log M\over \log\log M})$, the leading eigenvectors are polluted by the local properties of these heavy rows/columns and do not reveal the global
structure of the matrix/graph, which is precisely the desired information.

Fortunately, these heavy (empirical) rows/columns are the \emph{only} impediment to spectral concentration in the linear sample size regime.  Provided all rows/columns with observed weight significantly more than $d$ are zeroed out, spectral concentration prevails. This simple idea of taming the heavy rows/columns was first
introduced by~\cite{friedman1989second}, and analyzed in~\cite{feige2005spectral}
and many other works. Recently in~\cite{le2015sparse} and~\cite{le2015concentration}, the authors provided clean and clever
proofs to show that \emph{any} manner of ``regularization''---removing entries from the heavy rows/columns until their row/column sums are bounded---essentially leads to the desired spectral
concentration for the adjacency matrix of random graphs whose row/column sums are roughly uniform in
expectation.

The challenge of applying this regularization approach in our more general setting is that the row/column expectations of $C$ might be extremely non-uniform.  If we try to ``regularize'', we will not know whether we are removing entries from rows that have small expected sum but happened to have a few extra entries, or if we are removing entries from a row that actually has a large expected sum (in which case such removal will be detrimental).

Our approach is to partition the vocabulary $\mc M$ into bins that have roughly uniform marginal probabilities, corresponding to partitioning the rows/columns into sets that have roughly equal (empirical) counts.  Restricting our attention to the diagonal sub-blocks  of $\mbb B$ whose rows/columns consist of indices restricted to a single bin,  the expected row and column sums are now roughly uniform.  We can regularize (by removing abnormally heavy rows and columns) from each diagonal block
separately to restore spectral concentration on each of these sub blocks.  Now, we can apply truncated SVD to each diagonal sub block, recovering the column span of these blocks of $\mbb B$. With the 
column spans of each bin, we can now ``stitch'' them together as a single large projection matrix $P$ which has rank at most $R$ times the number of bins, and roughly contains the column span of $\mbb B$.  We then project a new count matrix, $C'$, obtained via a fresh partition of samples.  As the projection is fairly low rank, it filters most of the sampling noise, leaving an accurate approximation of $\mbb B$.

We summarize these basic ideas of  Algorithm~\ref{alg:rank-R} below. 
\begin{enumerate}
\item Given a batch of $N$ samples, group words according to the empirical marginal probabilities, so that in each bin consists of words whose (empirical) marginal probabilities, differ by at most a constant factor.
\item Given a second batch of $N$ samples, zeros out the words that have abnormally large empirical marginal probabilities comparing to the expected marginal probabilities of words in their bin.  Then consider the diagonal blocks of the empirical bigram counts matrix $C$, with rows and columns corresponding to the words in the same bin. We ``regularize'' each diagonal block in the empirical matrix by removing abnormally heavy rows and columns of the blocks, and then apply truncated SVD to estimate the column span of that diagonal block of $\mbb B$.
\item With a third batch of $N$ samples,  project the empirical count matrix into the ``stitched'' column spans recovered in the previous step which yields an accurate estimate of $\tx{Diag}(\rho)^{-1/2} \mbb B \tx{Diag}(\rho)^{-1/2}$ in spectral norm, where $\rho$ denotes the vector of marginal probabilities. Since the estimate is accurate in spectral norm \emph{after} scaling by the marginal probabilities, this spectral concentration of the scaled matrix easily translates into an $\ell_1$ error bounds for the un-scaled matrix $\mbb B$, as desired.
\end{enumerate}

%
%


%

 There are several potential concerns that arise in implementing the above high-level algorithm outline and establishing the correctness of the algorithm:
\begin{enumerate}
\item We do not have access to the exact marginal probabilities of each word. With a linear sample size, the recovered vector of marginal probabilities has only constant (expected) accuracy in $\ell_1$ norm.   Hence each bin, defined in terms of the empirical marginals, includes some non-negligible fraction of words with significantly larger (or smaller) marginal probabilities.  When directly applied to the empirical bins with such ``spillover'' words, the existing results of
  ``regularization'' in \cite{le2015concentration} do not lead to the desired concentration result.
\item When we restrict our analysis to a diagonal block corresponding to a single bin, we throw away all the sample
  counts outside of that block. This greatly reduces the effective sample size, since a significant fraction of a word's marginal probability might be due to co-occurrences with words outside of its bin.  It is not obvious
  that we retain enough samples in each diagonal block to guarantee meaningful estimation.   [If the mixing matrix $\mbb W$ in $\mbb B = \mbb P \mbb W \mbb P^\top$ is not p.s.d., this effect may be sufficiently severe so as to render these diagonal blocks essentially empty, foiling this approach.]
\item Finally, even if the ``regularization'' trick works for each diagonal block, we need to extract the
  useful information and ``stitch'' together this information from each block to provide an
  estimator for the entire matrix, including the off-diagonal blocks. Fortunately, the p.s.d assumption of the mixing matrix $W$ ensures that sufficient information is contained in these diagonal blocks.
\end{enumerate}

\begin{algorithm}[t!] \nonumber
  \caption{ The algorithm to which Theorem~\ref{thm:rank-R1} applies, which recovers rank $R$ probability matrices in the linear data regime.}
  \label{alg:rank-R}
  \DontPrintSemicolon

  \tb{Input:} $3N$ i.i.d. samples from the distribution $\mbb B$ of dimension $M\times M$, where $N=O(\frac{M R^2}{w_{min}^2\epsilon^5})$
  \\
  (In each of the 3 steps, $B$ refers to an independent copy of the normalized count matrix $\frac{1}{N} C$.)
  \\
  \tb{Output:} Rank $R$ estimator $\wh{\mbb B}$ for $\mbb B$
  \\
  \begin{enumerate} [label =Step \arabic*.,labelindent=*]

  \item (\tb{Binning according to  the empirical marginal probabilities})

    Set $ \wh \rho_i = \frac{\sum_{j=1}^M (C_{i,j}+C_{j,i})}{2N}$.
    Partition the vocabulary $\mc M$ into: 
    $${\mc I}_0 = \lt\{i:\wh\rho_i<\frac{1}{N}\rt\}, \tx{ and } 
    {\mc I}_{k} = \lt\{i:\frac{e^{k-1}}{N}\le \wh\rho_i\le \frac{e^k}{N}\rt\}, \tx{  for } k= 1,\ldots, \log N.$$
   Sort the $M$ words according to $\wh \rho_i$ in ascending order. Define $\bar\rho_k = {e^{k+1} \over N}$. For each bin ${\mc I}_k$, if 
$|\mc I_k|< 20 e^{-\frac{3}{2}(k+1)}N$
   set $\bar\rho_k$ to be $0$. Let $k_0=4\log(\frac{c_0 R}{\epsilon \sqrt{w_{min}}})+16$,  for an absolute constant $c_0$ which will be specified in the analysis, and set $\bar \rho_k$ to be $0$ for all $k<k_0$. Define the following block diagonal matrix:
    \begin{align}
      \label{eq:def-Ds}
      D= \lt[
      \begin{array}[c]{cccc}
        \bar\rho_1^{1/2} I_{|\mc I_1|}& &
        \\
         & \ddots &
        \\
         &  &   \bar\rho_{\log N}^{1/2} I_{|\mc I_{\log N}|}
       \end{array}
      \rt].
    \end{align}
    \medskip

  \item (\tb{Estimate dictionary span in each bin})
    
    For each diagonal block $B_k = B_{{\mc I}_k\times {\mc I}_k}$, perform the following two steps:
    \begin{enumerate}
      [labelindent=*,leftmargin=*,rightmargin=\dimexpr\linewidth-15cm]

    \item (\tb{Regularization}): 
    \begin{itemize}
    \item If a row/column of $B$ has sum exceeding $2\bar \rho_k$, set the entire row/column to 0.
    \item
    If a row/column of $B_k$ has sum exceeding $\frac{2|\mc I_k|\bar \rho_k^2}{w_{min}}$, set the entire row/column to 0.
    \end{itemize}
    Denote the regularized block by $\wt B_k$.
    
    \item (\tb{$R$-SVD}): Define the $|{\mc I}_k| \times R$ matrix $V_{k}$ to consist of the $R$ top singular vectors of $\wt B_k$.
    \end{enumerate}

  \item (\tb{Recover estimate for $\wh{\mbb B}$ accurate in $\ell_1$}) \\
  Define the following projection matrix:
    \begin{align}
      \label{eq:def-prj-wh-V}
      P_{V} = \lt[
      \begin{array}[c]{ccc}
        P_{V_1}& &
        \\
        & \ddots  &
        \\
        & & P_{V_{\log M}}
      \end{array}
      \rt], \text{ where } P_{V_k}= V_k V_k^T.
    \end{align} 


    Let $\wh{\mbb B'}$ be the rank-$R$ truncated SVD of matrix $P_{V} D^{-1} B D^{-1} P_{V}$, and return $\wh{\mbb B} = D\wh{\mbb B}' D$.

    \end{enumerate}



  \BlankLine
\end{algorithm}

\subparagraph*{Acknowledgements.}
Sham Kakade acknowledges funding from the Washington Research Foundation for Innovation in Data-intensive Discovery, and the NSF Award CCF-1637360. Gregory Valiant and Sham Kakade acknowledge funding form NSF Award CCF-1703574. Gregory and Weihao's contributions were supported by NSF CAREER Award CCF-1351108, and a Sloan Research Fellowship.

\bibliographystyle{plain}
\bibliography{newBib}
\newpage

\section{Proof of Theorem~\ref{thm:rank-R1}}
\label{sec:rank-r-algorithm}
In this section, we examine each step of Algorithm~\ref{alg:rank-R} to prove
Theorem~\ref{thm:rank-R1}.  Throughout the analysis, we will assume that we have access to three independent batches of samples, each consisting of $Poi(N)$ independent draws from the distribution defined by $\mbb B$.  With all but inverse exponential probability $Poi(N)$ and $N$ deviate by  $o(N)$, and hence this assumption is without loss of generality (as, for example, for each batch we could subsample $Poi(N)$ samples from a set of $2N$).  

Additionally, in this section we prove Theorem~\ref{thm:rank-R1} for a constant probability of failure, $\delta = 1/3.$  To obtain the general result for any $\delta > 0$, the probability of success can trivially be boosted to $1-\delta$ while  increasing the sample complexity by a factor of $O(\log(1/\delta))$.  Specifically, this can be achieved by randomly partitioning the $N$ samples into $O(\log(1/\delta))$ sets, applying the probability of success $\ge 2/3$ result that we prove in this section for target error $\eps/3$ to each set separately to recover $\wh {\mbb B}_1,\ldots,\wh {\mbb B}_{O(\log(1/\delta))},$ and then returning a $\wh {\mbb B}_i$ s.t. a majority of the recovered distributions $\{\wh {\mbb B}_1,\ldots, \wh {\mbb B}_{O(\log(1/\delta))}\}$ have distance at most $2\eps/3$ from the returned $\wh {\mbb B}_i$. Given that each of the returned distributions has distance at most $\eps/3$ from the target distribution, independently with probability at least $2/3$, basic Chernoff bounds and the triangle inequality guarantee that with probability at least $1-\delta$, such a  $\wh {\mbb B}_i$ exists and has distance at most $\eps$ from the target distribution.
%

We will let $C_1, C_2,$ and $C_3$ denote the respective count matrices derived from these three independent batches of $Poi(N)$ samples, corresponding to the three main steps of the algorithm.  Throughout this section, when the context is clear, we drop the subscripts and simply refer to the relevant matrix as $C$.

%

\subsection{Binning}
Let $\wh \rho_i$ denote the empirical marginal probability of the $i$th word (scaled by $N$, rather than the actual sample size of $Poi(N)$): $\frac{\sum_{j=1}^M (C_{i,j}+C_{j,i})}{2N}$.  We partition the vocabulary $\mc M$ according to the marginal probabilities, $\wh \rho$, 
\begin{align*}
  {\mc I}_0 = \lt\{i:\wh \rho_i < {1\over N} \rt\}, \quad{\mc
    I}_{k} = \lt\{i: {e^{k-1}\over N}\le \wh\rho_i< {e^{k}\over N}\rt\}, \tx{ for } k= 1, \ldots,\log N.
\end{align*}

Since we perform this binning based on the empirical probabilities, there will likely be some words whose true probability are significantly greater (or less than) than the average probability of words in their bin.  The words whose true probabilities are too light will not be an issue for us, though we will need to carefully consider the words that are too heavy.  To this end, for each bin, $k$, we define the set of ``spillover words'', ${\mc J}_k \subset \mc I_k$, to be those words whose true marginal probability exceeds the threshold $\bar \rho_k = e^{k+1}/N.$   The remaining words in the $i$th bucket, which we denote by ${{\mc L}_k} = {\mc I}_k \backslash { \mc J}_k$ will be referred to as the ``good words'' whose true marginal probabilities are at most $e^{k+1}/N.$

The following easy proposition argues that the total mass of the ``spillover'' words is small, across all bins.
\begin{proposition}[Spillover mass is small across all bins]\label{prop:small-spillover}
With probability $1-o(1)$, for all empirical bins ${\mc I}_k$, the spillover probability $\sum_{i\in \mc J_k} \rho_i \le e^{-e^{k-2}},$ and the sum of squares of the spillover probabilities, $\sum_{i\in \mc J_k} \rho^2_i \le \frac{e^{-e^{k-2}}}{N}$.
\end{proposition}

\begin{proof}
First we argue that with probability $1-o(1)$, no word with marginal probability $\rho_i \ge \frac{4 \log N}{N}$ will have empirical probability $\wh \rho_i\leq \frac{\rho_i}{e}$; namely, no heavy word will be a ``spillover word''. This follows immediately from standard tail bounds on Poisson random variables (Proposition~\ref{prop:chernoff-poisson}), and a union bound over the $M < N$ words.   


Next, for each bucket, $\mc I_k$, we will show that with probability at least $1-o(1/N)$, the probability mass of its spillover words is bounded by $e^{-e^{k-2}}$. For each bucket $k$, the total spillover mass can be written as: $\sum_{i : \rho_i  \in [\bar \rho_k, \frac{4 \log N}{N}]} \rho_i\tx{Ber}(p_i)$, where $p_i$ is the probability that word $i$ fall into bucket $k$. By Bernstein's inequality:
  \begin{align*}
    \Pr(\sum_{i:\rho_i \in [\bar \rho_k, \frac{4 \log N}{N}]} \rho_i\tx{Ber}(p_i)-\sum_{i:\rho_i \in[\bar \rho_k, \frac{4 \log N}{N}]} \rho_ip_i > t)\le \tx{exp}(-{ t^2\over {\sum_{i:\rho_i \in [\bar \rho_k, \frac{4 \log N}{N}]} \rho_i^2p_i} + \frac{4 \log N}{N} t}).
  \end{align*}
  
Leveraging the tail bound on Poisson random variables (Proposition~\ref{prop:chernoff-poisson}) to bound $p_i$, we have $\rho_i^2p_i\le \frac{1}{{e^{(k-1)e^k}}} {e^{-\rho_i}\rho_i^{e^k+2} }$. Taking the logarithm and computing the derivative over $\rho_i$ yields $-1+\frac{e^k+2}{\rho_i}$. Given $\rho_i\ge \bar{\rho}_k$, the derivative is always negative. Hence the maximum value of the previous upperbound of $\rho_i^2p_i$ is achieved when $\rho_i$ is equal to $\bar{\rho}_k$. Since the total number of $i$, such that $\rho_i\ge \bar{\rho}_k$, is at most $\frac{1}{\bar{\rho}_k}$, the term ${\sum_{i:\rho_i \ge\bar \rho_k} \rho_i^2p_i}$ is upper bounded by $\frac{e^{(2-e)e^{k}+k+1}}{N}$.  The expectation, $\sum_{i:\rho_i\ge\bar \rho_k} \rho_i p_i,$ can also be shown to be bounded by $e^{(2-e)e^k}$. Let $t=e^{-e^{k-2}}$, given $e^k\leq 4\log N$, the ratio $\frac{t^2}{{\sum_{i:\rho_i\ge\bar \rho_k} \rho_i^2\lambda_i} + \frac{4 \log N}{N}t}$ will be $\omega(\log N)$. Hence with probability $o(\frac{1}{N})$, the $k$'th bucket has total spillover mass more than $e^{-e^{k-2}}$. With a union bound over all buckets, we show that with high probability, any bucket $k$ has spillover mass less than $e^{-e^{k-2}}$.

In exact analogy to the above proof of the claimed bound on the sum of the marginal probabilities of the spillover words, we can argue a similar upper bound for the sum of squares of probabilities of the spillover words (i.e. $\sum_{i\in \mc J_k} \rho_i^2$), by setting $t=\frac{e^{-e^{k-2}}}{N}$.

\end{proof}

\subsection{Spectral concentration in diagonal blocks}

Under the assumption that $\mbb W$ is a p.s.d. matrix, we define the ${M\times R}$ matrix $\mbb B^{sqrt}$ to be: $$\mbb B^{sqrt} = \mbb P \mbb W^{1/2}$$
We use $B_k$ as shorthand for $B_{{\mc I}_k\times {\mc I}_k}$ which is the $k$'th diagonal block of the empirical probability matrix, and use $\mbb B_k$ to denote $\mbb B_{\mc I_k \times \mc I_k}$ which is the $k$'th diagonal block of the true probability matrix. Similarly, we define the matrix $B^{sqrt}$ restricted to bin ${\mc I}_k$ as $B_{k}^{sqrt}$ and the matrix  $\mbb B^{sqrt}$ restricted to bin ${\mc I}_k$ as $\mbb B_{k}^{sqrt}$.

In Lemma~\ref{lem:diag-block-concen}, we argue that the regularized block corresponding to each bin concentrates to the corresponding portion of the underlying probability matrix. The main difficulty in showing this comes from the non-uniformity of the entries, caused by the spillover words.  In~\cite{le2015concentration}, a similar type of concentration is established (which we restate as Lemma~\ref{lem:vershynin} in the Appendix), though the quality of the concentration degrades with the ratio of the maximum entry-wise probability to the average probability.  The relatively large probabilities due to the spillover words in our context would yield a result that is a super-constant factor suboptimal.  

In order to deal with the (relatively small amount of) spillover words, we regularize each block matrix twice, once for the words whose overall row/column sum are too big for the bin, and once for the words whose in-block row/column sums are too big. These two types of regularization correspond to the two bulleted steps in part (a) of Step 2 of the algorithm.  We now apply a coupling argument to show that after such regularization, the remaining matrix is just like a regularized ``uniform'' matrix whose spectral concentration property can then be guaranteed in the same spirit as the main theorem in~\cite{le2015concentration}. 

The high level idea is as follows: given the observed count matrix, $C$, we will selectively subsample entries to obtain the matrix $C^R$, which has the property that for all $i,j \in \mc I_k$ the expectation of the $(i,j)$th entries of $C^R$ are at most a constant factor larger than the average entries of the corresponding block.  Specifically, it will be the case that $C^R_{i,j} = O({N \bar \rho_k^2 \over w_{min}})$.  (Recall that $w_{min}$ is defined as $\min_i \sum_j \mbb W_{i,j}$, where $\mbb B = \mbb P \mbb W \mbb P^T.$)

Next, we show that the regularized count matrix $\wt C_{\mc I_k\times \mc I_k}$ computed by the algorithm can be obtained from $C^R$ by zeroing out a small number of rows and columns, which is sufficient to guarantee that $\|\wt C_{\mc I_k\times \mc I_k}- \E C_{\mc I_k \times \mc I_k}^R \| \le O(\sqrt{{N |\mc I_k|\bar \rho_k^2 \over w_{min}}})$ by Lemma~\ref{lem:vershynin}.  The remaining piece of the proof is to show that $\|\E C^R_{\mc I_k \times \mc I_k}  -\E C_{\mc I_k \times \mc I_k}\|\le O(\sqrt{{N |\mc I_k|\bar \rho_k^2 \over w_{min}}})$, which holds because most of the entries of both matrices correspond to ``good'' (not spillover) words, and these entries are identical in the two matrices.

This establishes the spectral concentration of $\wt C_k = N \wt B_k$; finally, in Corollary~\ref{cor:diag-block-concen}, we show that given this concentration, the span of the top $R$ singular vectors of this matrix is close to the column span of the underlying matrix, $\mbb B_k$.

We now formalize the above high-level outline.
\medskip

The following lemma bounds the number of rows and columns of each block that are zeroed out via the regularization step (Step 2) of the algorithm.  This lemma is intuitively clear, and the proof is via a series of tedious Chernoff bounds.

\begin{lemma}\label{lem:num-reg}
With probability $1-o(1)$, for all $k$, less than $\frac{10 w_{min}}{32 \bar \rho_k^2 N}$ rows and column in the $k$-th bucket are zeroed out by Step 2 of the algorithm.
\end{lemma}

\begin{proof}
The columns/rows that are zeroed out can be partitioned into two sets, the ones corresponding to good words and ones corresponding to spillover words.  There are at most $2\frac{\|\rho_{\mc J_k}\|_{1}}{\bar \rho_k}$ rows/columns corresponding to spillover words, as the numerator is the total mass of such words, and the denominator is a lower bound on the mass of each word. In order to show that $2\frac{\|\rho_{\mc J_k}\|_{1}}{\bar \rho_k}\le \frac{5 w_{min}}{32 \bar \rho_k^2 N}$, by Proposition~\ref{prop:small-spillover}, it suffices to show that $\|\rho_{\mc J_k}\|_{1}\le e^{-e^{k-2}}\le \frac{5}{64} w_{min} e^{-k-1}$. 
Taking the logarithm of both sides yields $-e^{k-2}\le \log(5/64)+\log w_{min}-(k+1)$ which is equivalent to $\log(5/64)+\log w_{min}-(k+1)+e^{k-2}\ge 0$. Notice that the left hand side is monotonically increasing when $k>3$, hence we only need to verify the inequality by plugging in the lowerbound of $k$ (i.e. $\log\log (\frac{2}{w_{min}})+6$) which yields 
$$
\log(5/64)+\log w_{min}-(\log\log (\frac{2}{w_{min}})+7)+e^4 \log \frac{2}{w_{min}}
$$
whose non-negativity can be verified by direct calculation. Hence, with the claimed probability, no bin has more than $\frac{5 w_{min}}{32 \bar \rho_k^2 N}$ spillover rows and columns that are zeroed out.

We now consider the rows/columns corresponding to good words.  Specifically, for a good word, we show that the probability that the corresponding row sum of $C_{\mc I_k \times \mc I_k}$ exceeds $2\bar \rho_k N$ is at most $e^{-\frac{\bar \rho_k^2 |\mc I_k| N \log4 }{w_{min}}}$. Since the number of good words is at most $|\mc I_k|$, the probability that the number of zeroed out good word rows is bigger than $\frac{ w_{min}}{32 \bar \rho_k^2 N}$ can be upper bounded by $\Pr(\sum_{i=1}^{|\mc I_k|}X_i\ge  \frac{ w_{min}}{32 \bar \rho_k^2 N})$ where $X_i\sim Bernoulli(e^{-\frac{\bar \rho_k^2 |\mc I_k| N \log4 }{w_{min}}})$. 
Given that the algorithm only keeps bins with $|\mc I_k|\ge 20 e^{-\frac{3}{2}(k+1)}N$, the probability that $X_i=1$ is smaller than $e^{- e^{\frac{1}{2}(k+1)}20 \log 4 \over w_{min}}$. In the case that $k\ge 2\log\log N$, by a union bound we get that with high probability, all $X_i=0$. Otherwise $k <2\log\log N$, we have $\E[\sum_{i=1}^{|\mc I_k|}X_i] \le |\mc I_k|(e^{-\frac{\bar \rho_k^2 |\mc I_k| N \log4 }{w_{min}}})$ which is monotonically decreasing for $|\mc I_k|> \frac{w_{min}}{\bar \rho_k^2 N \log4 } = \frac{w_{min}}{ \log4 }e^{-2(k+1)}N$. Since $20 e^{-\frac{3}{2}(k+1)}N>\frac{w_{min}}{ \log4 }e^{-2(k+1)}N$, we can simply plug in $|\mc I_k| = 20 e^{-\frac{3}{2}(k+1)}N$ to yield a worst case bound: $\E[\sum_{i=1}^{|\mc I_k|}X_i]\le 20 e^{-\frac{3}{2}(k+1)}Ne^{- e^{\frac{1}{2}(k+1)}20 \log 4 \over w_{min}}$. 

We will now bound the probability that $\sum_{i=1}^{|\mc I_k|}X_i$ exceeds its expectation by more than a factor of $\delta = \frac{ w_{min}}{32 \bar \rho_k^2 N}/20 e^{-\frac{3}{2}(k+1)}Ne^{- e^{\frac{1}{2}(k+1)}20 \log 4 \over w_{min}} = \frac{w_{min}}{640}e^{-\frac{1}{2}(k+1)} e^{ e^{\frac{1}{2}(k+1)}20 \log 4 \over w_{min}}>100$ for all $k>k_0$.  Since the $X_i$'s are independent, we may apply a Chernoff bound to yield $\Pr(\sum_{i=1}^{|\mc I_k|}X_i\ge  \frac{ w_{min}}{32 \bar \rho_k^2 N}) \le e^{-\frac{w_{min}N}{100e^{2(k+1)}}}$. Further, applying a union bound for all $k<2\log\log N$, shows that with probability $1-o(1)$, no bin has more than $\frac{ w_{min}}{32 \bar \rho_k^2 N}$ good rows and columns zeroed out. The number of zeroed out rows/columns due to the sum being bigger than $\bar \rho_k N$ in matrix $C$ follows a similar argument. This shows that the total number of zeroed out rows/columns is less than $\frac{10 w_{min}}{32 \bar \rho_k^2 N}$ with the claimed probability.
\end{proof}

\begin{lemma}[Spectral concentration in each diagonal block]
  \label{lem:diag-block-concen}
  With probability $1-|\mc I_k|^{-r}-o(1)$, the regularized matrix $\wt B_{\mc I_k\times \mc I_k}$ from the second step of the algorithm concentrates to the underlying probability matrix $B_{\mc I_k\times \mc I_k}$:
  $$\|\wt B_{\mc I_k\times \mc I_k}-\mbb B_{\mc I_k\times \mc I_k}\| = O\lt(r^{3/2}\sqrt{|\mc I_k| \bar \rho_k^2 \over N w_{min}} \rt).$$
\end{lemma}

\begin{proof}

Recall that $C$ is the original counts matrix, and $\wt C$ is the counts matrix after the two types of regularization. Our proof will hinge on constructing a matrix $C^R$ from $C$ and the marginal vector $\rho$ that has the following two properties:
\begin{itemize}
\item $\wt C_{\mc I_k \times \mc I_K}$ can be obtained from $C^R_{\mc I_k \times \mc I_k}$ by removing few rows and columns.
\item While $C^R$ is a function of both $\rho$ and the random variable $C$, the (marginal) distribution of $C^R$ will have each entry drawn independently from Poisson distributions, where the expectations of the Poisson distributions corresponding to elements of the same diagonal block, $\mc I_k$, are within constant factors of each other.
\end{itemize}

The idea behind this construction is to subsample the rows and columns of $C$ which correspond to spillover words so as to reduce their effective marginals to the level of those of the good words in the same bin. However, we need to make sure that we only subsample the entries that will be zeroed out in $\wt C$, because otherwise $\wt C$ can not be obtained from $C^R$ by zeroing out rows and columns. 
The procedure of sampling $C^C$ is described as following: For a spillover column, we sample the ``target'' column sum from $Poi(\rho_i N)$, and if the sampled column sum is smaller than the bucket threshold $\bar \rho_k N$, we will not modify that column. If, however, the column sum is bigger than the bucket threshold, we redraws the row sum from a designed distribution such that the resulting distribution corresponds to a Poisson distribution with reduced mean. 

Imagine obtaining the counts matrix $C$ along with the vector of true word marginal probabilities, $\rho$. Let $Poi(\lambda,x)=\frac{\lambda^x e^{-\lambda}}{x!}$ denote the p.m.f. of a Poisson distribution with mean $\lambda$ evaluated at integer $x$. For any word $i\in \mc I_k$ for which $\rho_i\ge 4\bar \rho_k$ and the column sum $\sum_{j=1}^M C_{j,i} \ge 2\bar \rho_k N$, we will reduce the entries in the corresponding column of $C$ as follows.  First, with probability $\frac{\max\left(0, Poi\left(\rho_i N,\sum_{j=1}^M C_{j,i}\right)-Poi\left(\bar \rho_k N,\sum_{j=1}^M C_{j,i}\right)\right)}{Poi\left(\rho_i N,\sum_{j=1}^M C_{j,i}\right)}$, we draw a sample $x_i$ from the distribution with p.m.f. $\frac{max\left(0, Poi(\bar \rho_k N,x)-Poi(\rho_i N,x)\right)}{Z}$, where $Z$ is the appropriate  normalization factor, otherwise we set $x_i=\sum_{j=1}^M C_{j,i}$. Note that since $Poi(\bar \rho_k, 2\bar \rho N)> Poi(\rho_i,2\bar \rho N)$, $x_i$ is indeed distributed as $\tx{Poi}(\bar \rho_k N)$ and is always less than or equal to $\sum_{j=1}^M C_{j,i}$ by construction. 

We now reduce the column sum of the $i$th column of $C$ until the sum is $x_i$ by selecting a subset of $x_i$ counts uniformly  at random from the $\sum_{j=1}^M C_{j,i}$ counts to remain.  Let $C^{C}$ be the final result of this operation and denote the set of indices of the columns that were modified by $S_C$. By construction, for any column $i\in \mc I_k$ with $\rho_i \ge 4\bar \rho_k$, the entry $C^{C}_{j,i}$ follows a Poisson distribution $\tx{Poi}(\frac{\bar \rho_k}{\rho_i} \E C_{j,i})$. 

Given the matrix $C^C$ along with the vector $\rho'$ which is the ``row marginal'' of the modified matrix $C^C$, specifically $\rho'_i = \frac{\sum_{j=1}^M \E C^C_{i,j}}{N}$. Notice that for any $i\in \mc I_k$, $\rho_i - \rho'_i$ is fairly small, specifically less than $\frac{\rho_i {\|\rho_{\mc J_k}\|}_1}{w_{min}}$ where ${\|\rho_{\mc J_k}\|}_1$ is the spillover probability of bucket $k$. Use the fact that $k >4\log(\frac{C\sqrt{R}}{\epsilon \sqrt{w_{min}}})+16 > \log \log (\frac{2}{w_{min}})+6$, the spillover probability is less than $e^{-e^{k-2}}\le \frac{w_{min}}{2}$ and hence $(\rho_i - \rho'_i)\le \frac{1}{2}\rho_i$. In analogy to what we did to the columns, we resample spillover rows to make the marginal distribution small. Specifically, for any word $i\in \mc I_k$ that $\rho'_i\ge 4\bar \rho_k$ and the row sum $\sum_{j=1}^M C^C_{i,j} \ge 2\bar \rho_k N$, with probability $\frac{\max(f_{\rho'_i N}(\sum_{j=1}^M C^C_{i,j})-f_{\bar \rho_k N}(\sum_{j=1}^M C^C_{i,j}),0)}{f_{\rho'_i N}(\sum_{j=1}^M C^C_{i,j})}$, we draw a sample $y_i$ from the distribution with p.m.f. $\frac{max(f_{\bar \rho_k N}(x)-f_{\rho'_i N}(x), 0)}{Z}$, where $Z$ is the proper normalization factor, otherwise let $y_i=\sum_{j=1}^M C^C_{i,j}$. $y_i$ is distributed as $\tx{Poi}(\bar \rho_k N)$ and always less than or equal to $\sum_{j=1}^M C^C_{i,j}$ by construction. 
We further remove $\sum_{j=1}^M C^C_{i,j}-y_i$ ones from the $i$th row of matrix $C^C$ randomly in analogy to what we did to the columns. Let $C^{R}$ be the final result of our operation and denote the set of indices of the rows that we modified by $S_R$. By construction, for any row $i\in \mc I_k$ and $\rho'_i >4\bar \rho_k$, entry $j$ of row $i$, $C^{R}_{i,j}$, follows a Poisson distribution $\tx{Poi}(\frac{\bar \rho_k}{\rho'_i} \E C^C_{j,i})$. 

One desired property of random matrix $C^R$ is that the expectation of each entry within $\mc I_k \times \mc I_k$ is pretty well bounded. For any  $i,j\in \mc I_k$, due to the way we construct $C^R$, only if $\rho_i<4\bar\rho_k$ and $\rho_j>8\bar\rho_k$ will we keep the original entry, which guarantees $\E C^R_{i,j}\leq \frac{32 \bar \rho_k^2 N}{w_{min}}$. Denote $\wt C^R$ to be the matrix $C^R$ with columns $S_C$ and rows $S_R$ being zeroed out. 
Given the fact that $S_C$ and $S_R$ contain only the rows/columns whose sum are larger than $2 \bar\rho_k N$, which are zeroed out in $\wt C_{\mc I_k \times \mc I_k}$. Hence $\wt C_{\mc I_k \times \mc I_k}$ can also be thought as the random matrix $C^R$  with rows and columns zeroed out. 
By Lemma~\ref{lem:num-reg}, not too many rows/columns are zeroed out(less than $\frac{10 w_{min}}{32 \bar \rho_k^2 N}$), Lemma~\ref{lem:vershynin} immediately imply that the spectral discrepancy between $\wt C_{\mc I_k \times \mc I_k}$ and $\E C^R_{\mc I_k \times \mc I_k}$ is bounded by $O(r^{3/2}\sqrt{\frac{|\mc I_k| N \bar \rho_k^2}{w_{min}}})$.

The final step is to show $\|\E C^R_{\mc I_k \times \mc I_k} - \E C_{\mc I_k \times \mc I_k}\|$ is also in the order of $O(\sqrt{\frac{|\mc I_k| N \bar \rho_k^2}{w_{min}}})$. On the good words, $\E C^R$ is the same as $\E C$. On the rows and columns involving spillover words, $\E C^R$ is always less than or equal to $\E C$. Write $\E C$ as $N\mbb B$ yields: $\|\E C^R_{\mc I_k \times \mc I_k} - \E C_{\mc I_k \times \mc I_k}\| \le N(2\|\mbb B_{J_k\times L_k}\|+\|\mbb B_{J_k\times J_k}\|)\le N(2\|\mbb B_{J_k\times L_k}\|_{F}+\|\mbb B_{J_k\times J_k}\|_{F})$. The Frobenius norm can be bounded using the fact that the sum of squares of spillover words  marginals is small (Proposition~\ref{prop:small-spillover}): with high probability, $\|\mbb B_{J_k\times L_k}\|_{F}\le \sqrt{\sum_{i\in J_k} \frac{\rho_i^2 \bar\rho_k^2}{w_{min}^2}|L_k|} \le \sqrt{\frac{e^{-e^{k-2}} \bar\rho_k^2}{N w_{min}^2}|L_k|}$, $\|\mbb B_{J_k\times J_k}\|_{F}\le \sqrt{\sum_{i\in J_k, j \in J_k} \frac{\rho_i^2 \rho_j^2}{w_{min}^2}} \le \sqrt{\frac{e^{-2e^{k-2}}}{N^2 w_{min}^2}}$. With the assumption that $k> \log \log (\frac{2}{w_{min}})+3$, $\|\E C^R_{\mc I_k \times \mc I_k} - \E C_{\mc I_k \times \mc I_k}\|=O(\sqrt{\frac{|\mc I_k| N \bar \rho_k^2}{w_{min}}})$.

Finally, we have $\|\wt C_{\mc I_k \times \mc I_k} - \E C_{\mc I_k \times \mc I_k}\| \le \|\wt C_{\mc I_k \times \mc I_k} - \E C^R_{\mc I_k \times \mc I_k}\|+\|\E C^R_{\mc I_k \times \mc I_k} - \E C_{\mc I_k \times \mc I_k}\| = O(r^{3/2}\sqrt{\frac{|\mc I_k| N \bar \rho_k^2}{w_{min}}})$ and hence $\|\wt B_{\mc I_k \times \mc I_k} - \E B_{\mc I_k \times \mc I_k}\|= O(r^{3/2}\sqrt{\frac{|\mc I_k|\bar \rho_k^2}{w_{min}N}})$ as desired.
\end{proof}

\begin{corollary}\label{cor:diag-block-concen}
Let the columns of the $|\mc I_k| \times R$ matrix $V_k$ be the leading $R$ singular vectors of regularized block $\wt B_{\mc I_k\times \mc I_k}$, Define $P_{V_{k}} = V_k V_k^\top$. Then with probability $1-|\mc I_k|^{-r}- o(1)$, we have
  \begin{align}
    \label{eq:bin-concen-sqrt}
    \| P_{ V_{k}} {\mbb B}_{k}^{sqrt} - {\mbb B}_{k}^{sqrt}\| =
    O\lt(r^{3/4}\lt(|\mc I_k| \bar \rho_k^2 \over N w_{min} \rt)^{1/4}\rt).
  \end{align}
\end{corollary}
\begin{proof}
By the triangle inequality, we have $\| P_{V_k} {\mbb B}_{k} P_{V_k} - {\mbb B}_{k}\| \le \| P_{V_k} ( \mbb B_{k} - \wt B_{k}) P_{ V_k}\| + \| P_{V_k} \wt B_{k} P_{ V_k} - \wt B_k\| +\| \wt B_{k}-\mbb B_k\|$. The first term is bounded by $\|\mbb B_{k} - \wt B_{k}\|$ since $P_{V_k}$ is an orthogonal matrix. The second term is bounded by $\|\mbb B_{k} - \wt B_{k}\|$ since $P_{V_k} \wt B_{k} P_{ V_k}$ is the best rank $R$ approximation of $\wt B_k$ and hence must be a better approximation than $\mbb B_k$ which is also rank $R$. Hence $\| P_{V_k} {\mbb B}_{k} P_{V_k} - {\mbb B}_{k}\| \le 3\|\wt B_k - \mbb B_k\| = O(\sqrt{\frac{|\mc I_k|\bar \rho_k^2}{w_{min}N}})$ by Lemma~\ref{lem:diag-block-concen}. Finally, applying Lemma~\ref{claim:sqrt-perturbation} we have $\| P_{V_{k}} {\mbb B}_{k}^{sqrt} - {\mbb B}_{k}^{sqrt}\| \le \sqrt{\|  P_{V_k} {\mbb B}_{k} P_{V_k} - {\mbb B}_{k}\|}=O\lt(\lt(|\mc I_k| \bar \rho_k^2\over N w_{min}\rt)^{1/4}\rt)$.
\end{proof}


\subsection{Low rank projection}
In Step 3 of Algorithm~\ref{alg:rank-R}, we ``stitch'' together the subspaces $\{V_k\}$ recovered in Step 2, to get an estimate for the column span of the entire matrix.

Define the diagonal matrix $D_S$ of dimension $M\times M$ to be:
\begin{align*}
      D= \lt[
      \begin{array}[c]{cccc}
        \sqrt{\bar\rho_1} I_{|\mc I_1|}& &
        \\
         & \ddots &
        \\
         &  &   \sqrt{\bar\rho_{\log N}} I_{|\mc I_{\log N}|}
       \end{array}
      \rt].
\end{align*}
Define $P_{V}$ to be the block diagonal projection matrix which projects an
$M\times M$ matrix to a subspace $V$ of dimension at most $R\log N$:
\begin{align*}
  P_{V} = \lt[
  \begin{array}[c]{ccc}
    P_{V_1}& &
    \\
    & \ddots  &
    \\
    & & P_{V_{\log N}}
  \end{array}
  \rt].
\end{align*}
Now consider the empirical scaled counts, $B=C/N$ derived from the 3rd batch of samples.  The following proposition, Proposition~\ref{prop:noisefilter}, shows that after being projected to the $R\log N$ dimensional subspace, the noise reduces substantially. Of course we also need to argue that the signal is still preserved, which we do in Proposition~\ref{prop:subspace}. Together, these two propositions yield Corollary~\ref{cor:estimator}, which argues that the projection yields a good estimator of the scaled probability matrix when applied on the sample matrix.

\begin{proposition}[Projection Reduces Noise]\label{prop:noisefilter}
With probability larger than $\frac{3}{4}$, $\|P_V D^{-1} (B-\mbb B) D^{-1} P_V\| \le O(\frac{R\log^2(N) }{\sqrt{w_{min}N}})$
\end{proposition}
\begin{proof}
Let us determine the variance of one entry of matrix $D^{-1}(B-\mbb B)D^{-1}$: $\tx{Var}[\frac{B_{i,j}-\mbb B_{i,j}}{\bar \rho_{k(i)} \bar \rho_{k(j)}}] \le \frac{\rho_i \rho_j}{w_{min}\bar \rho_{k(i)} \bar \rho_{k(j)}N}$. The variance term is super constant only in case of spillover and we showed that in Proposition~\ref{prop:small-spillover} that the spill over will not happen for $\rho_i \ge 4\frac{\log(N)}{N}$. Hence $\frac{\rho_i}{\bar \rho_{k(i)}}$ is at most $\log N$ and $\tx{Var}[\frac{B_{i,j}-\mbb B_{i,j}}{\bar \rho_{k(i)} \bar \rho_{k(j)}}] \le \frac{\log^2(N)}{w_{min}N}$. By Proposition~\ref{thm:sub-exponential-facts}, with the claimed probability $\|D^{-1} P_V (B-\mbb B) P_V D^{-1}\| \le O(\frac{R\log^2(N) }{\sqrt{w_{min}N}})$.
\end{proof}

\begin{proposition}[Projection Preserves Signal]\label{prop:subspace}
With probability $1-1/256-o(1)$, $\|P_V D^{-1} \mbb B D^{-1} P_V - D^{-1} \mbb B D^{-1}\| \le O((\frac{1}{e^{k_0} w^2_{min}})^{1/4})$
\end{proposition}
\begin{proof}
\begin{align*}
&\| D^{-1}P_V \mbb B P_V D^{-1} - D^{-1}\mbb BD^{-1} \| \\
=& \|(D^{-1}P_V \mbb B^{sqrt} - D^{-1}\mbb B^{sqrt})\mbb B^{sqrt} P_V D^{-1} + D^{-1}\mbb B^{sqrt}(\mbb B^{sqrt} P_V D^{-1}- \mbb B^{sqrt} D^{-1})\|\\
\le& \|(D^{-1}P_V \mbb B^{sqrt} - D^{-1}\mbb B^{sqrt})\| \|\mbb B^{sqrt} P_V D^{-1}\| + \|D^{-1}\mbb B^{sqrt}\| \|(\mbb B^{sqrt} P_V D^{-1}- \mbb B^{sqrt} D^{-1})\|\\
\le& 2 \|(D^{-1}P_V \mbb B^{sqrt} - D^{-1}\mbb B^{sqrt})\| \|\mbb B^{sqrt} D^{-1}\| 
\end{align*}
We can bound the second term as: $\|\mbb B^{sqrt} D^{-1}\| \le \sqrt{\| D^{-1} \mbb B D^{-1}\|} \le \sqrt{\| D^{-1} \mbb B D^{-1}\|_F} \le ( \sum_{i,j} \frac{\rho^2_i \rho^2_j }{w_{min} \bar \rho_{k(i)} \bar \rho_{k(j)}})^{1/4}$. Denote $\mc J, \mc L$ as the set of all spillover words and good works, respectively. The summation can be partitioned into 3 parts: 
\begin{enumerate}
\item For $i,j$ such that both $i$ and $j$ are good words, the summation is upper bounded by $(\sum_{i\in \mc L, j\in  \mc L}\frac{\rho_i \rho_j}{w_{min}}) \le \frac{1}{w_{min}}$.
\item For $i,j$ such that either $i$ or $j$ is good word, the summation is upper bounded by $\sum_{i,j} \frac{\rho^2_i \rho^2_j }{w_{min} \bar \rho_{k(i)} \bar \rho_{k(j)}} \le \sum_{i\in \mc J, j\in \mc L}\frac{\rho_i^2 \rho_j}{w_{min} \rho_{k(i)}} \le \sum_{i\in \mc J}\frac{\rho_i^2}{w_{min} \bar \rho_{k(i)}}$. $\bar \rho_{k(i)}$ must be at least $\frac{e^{k_0}}{N}$. With the bound for sum of squares of spillover marginals in Proposition~\ref{prop:small-spillover} we have $\sum_{i\in \mc J}\frac{\rho_i^2}{w_{min} \bar \rho_{k(i)}}\le \frac{\sum_{k=k_0}^{\log(N)} e^{-e^{k-2}}}{w_{min}e^{k_0}}$. Applying the assumption that $k_0 \ge \log(\log(2/w_{min}))+3$, we have $\sum_{i\in \mc J, j\in \mc L}\frac{\rho_i^2 \rho^2_j}{w_{min} \rho_{k(i)}\rho_{k(j)}}\le \frac{1}{w_{min}}$
\item For $i,j$ such that both $i$ and $j$ are spillover words, the summation is upper bounded by $$\sum_{i\in \mc J, j\in \mc J}\frac{\rho_i^2 \rho_j^2}{w_{min} \bar \rho_{k(j)}\bar \rho_{k(i)}} \le \frac{(\sum_{k=k_0}^{\log(N)} e^{-e^{k-2}})^2}{w_{min} \bar \rho_{k_0}^2}.$$ With the assumption that $k_0 \ge \log(\log(2/w_{min}))+3$ and $\bar \rho_{k(i)} >\frac{e^{k_0}}{N}$, we have $\sum_{i\in J, j\in J}\frac{\rho_i^2 \rho_j^2}{w_{min} \bar \rho_{k(j)}\bar \rho_{k(i)}}\le \frac{1}{w_{min}}$
\end{enumerate}
Combining the 3 parts yields $\|\mbb B^{sqrt} D^{-1}\|=O(\frac{1}{w_{min}^{1/4}})$.

The first term $\|(D^{-1}P_V \mbb B^{sqrt} - D^{-1}\mbb B^{sqrt})\|$ will be bounded using Corollary~\ref{cor:diag-block-concen}. Matrix $D^{-1}(P_V \mbb B^{sqrt} - \mbb B^{sqrt})$ is concatenated by matrices  $\bar\rho_k^{-1/2} (P_{V_k} \mbb B_k^{sqrt} - \mbb B_k^{sqrt}),k =1, \ldots, \log(N)$ and its spectral norm can be bounded as: $\| D^{-1} (P_{V} \mbb B^{sqrt} - \mbb B^{sqrt}) \| \le (\sum_{k=k_0}^{\log(n)} \bar\rho_k^{-1} \| P_{V_k} {\mbb B}_{k}^{sqrt} - {\mbb B}_{k}^{sqrt} \|^2)^{1/2}=( \frac{1}{\sqrt{w_{min}N}} \sum_{k=k_0}^{\log N} \sqrt{r_k^{3}|\mc I_k|})^{1/2}$ with probability $1-\sum_{k=k_0}^{\log N}|\mc I_k|^{-r_k}-o(1)$. In order to make it hold with large probability, we need to find appropriate $r_k$. Since we require $\mc |I_k|\ge 20e^{-\frac{3}{2}(k+1)}N$, we can simply let $r_k=10$ for all $k<\frac{2}{3}\log N$ which yields $\sum_{k=k_0}^{\frac{2}{3}\log N}|\mc I_k|^{-r_k}\le 1/256$. For the rest $k$, we set $r_k$ to be $\log N$ such that $|\mc I_k|^{-r_k}\le 1/N$(notice that we can assume $|\mc I_k|> k$ since otherwise the error caused by projection will be $0$) and hence $\sum_{\frac{2}{3}\log N}^{\log N} |\mc I_k|^{-r_k}\le \frac{\log N}{N}=o(1)$. Hence with probability at least $1-\frac{1}{256}-o(1)$, we have $\| D^{-1} (P_{V} \mbb B^{sqrt} - \mbb B^{sqrt}) \| \le ( \frac{1}{\sqrt{w_{min}N}} \sum_{k=k_0}^{\log N} \sqrt{r_k^{3}|\mc I_k|})^{1/2}$. Notice that $|\mc I_k| \le \frac{N}{e^k}$, we get $( \frac{1}{\sqrt{w_{min}N}} \sum_{k=k_0}^{\log N} \sqrt{r_k^3|\mc I_k|})^{1/2} = O((\frac{1}{e^{k_0} w_{min}})^{1/4})$.

Putting the bounds of the two terms together yields: $\|(D^{-1}P_V \mbb B^{sqrt} - D^{-1}\mbb B^{sqrt})\| \|\mbb B^{sqrt} D^{-1}\| =O((\frac{1}{e^{k_0} w^2_{min}})^{1/4})$
\end{proof}

\begin{corollary}\label{cor:estimator}
Let $\wh B'$ be the rank $R$ truncated SVD of matrix $P_V D^{-1} B D^{-1} P_V$. With probability larger than $\frac{2}{3}-o(1)$, $\|\wh B' - D^{-1}\mbb B D^{-1}\| =O(\frac{R\log^2(N) }{\sqrt{w_{min}N}} + (\frac{1}{e^{k_0} w^2_{min}})^{1/4})$
\end{corollary}
\begin{proof}
$\|\wh B' - D^{-1}\mbb B D^{-1}\| \le  \|\wh B' - P_V D^{-1} B D^{-1} P_V\|+\|P_V D^{-1} B D^{-1} P_V-D^{-1}\mbb B D^{-1}\|$. Given that $\wh B'$ is the best rank $R$ approximation of matrix $P_V D^{-1} B D^{-1} P_V$ and $D^{-1}\mbb B D^{-1}$ is also rank $R$, 
\begin{align*}
&\|\wh B' - P_V D^{-1} B D^{-1} P_V\|+\|P_V D^{-1} B D^{-1} P_V-D^{-1}\mbb B D^{-1}\| \\
\le& 2 \|P_V D^{-1} B D^{-1} P_V-D^{-1}\mbb B D^{-1}\|\\
\le& 2\|P_V D^{-1} (B-\mbb B) D^{-1} P_V\| + 2\|P_V D^{-1} \mbb B D^{-1} P_V - D^{-1} \mbb B D^{-1}\|.
\end{align*}
At this point with the bounds established in Proposition~\ref{prop:noisefilter} and Proposition~\ref{prop:subspace}, we have $\|\wh B' - D^{-1}\mbb B D^{-1}\|= O(\frac{R\log^2(N) }{\sqrt{w_{min}N}}+(\frac{1}{e^{k_0} w^2_{min}})^{1/4})$.
\end{proof}

\subsection{Completing the Proof of Theorem~\ref{thm:rank-R1}}\label{thm:main_tech}
Having established an accurate estimate of the scaled probability matrix under the operator norm error, our main theorem  can be proved directly with Cauchy-Schwartz, with the additional minor issue of bounding the error due to the rows/columns that were excluded because their marginal probabilities were too small. 

We begin by restating Theorem~\ref{thm:rank-R1} in the case that the desired failure probability $\ge 2/3$.  As noted at the beginning of this section, such a result can trivially be leveraged to yield success probability $1-\delta$ at the expense of increasing the sample size by a factor of $O(\log(1/\delta))$, for any $\delta > 0$.

\begin{theorem}
Let $\wh{B'}$ be the rank $R$ truncated SVD of matrix $P_V D^{-1} B D^{-1} P_V$ and $\wh{\mbb B}=D \wh{B'} D$. With probability at least $2/3-o(1)$:
$$
    \|\wh {\mbb B}-\mbb B\|_{\ell_1} \le \epsilon.
 $$
\end{theorem}
\begin{proof}
Apply Cauchy-Schwartz to the $\ell_1$ norm: 
\begin{align*}
\| \wh{\mbb B}-\mbb B\|_{\ell_1}=&\sum_{i,j: \bar \rho_{k(i)}\ne 0 \tx{ and } \bar \rho_{k(j)}\ne 0} |(\wh{\mbb B}-\mbb B)_{i,j}|\frac{1}{\sqrt{\bar \rho_{k(i)} \bar \rho_{k(j)}}}\sqrt{\bar \rho_{k(i)} \bar \rho_{k(j)}} + 
\sum_{i,j: \bar \rho_{k(i)}= 0 \tx{ or } \bar \rho_{k(j)}=0} |(\wh{\mbb B}-\mbb B)_{i,j}|\\
\le& \sqrt{\sum_{i,j: \bar \rho_{k(i)}\ne 0 \tx{ and } \bar \rho_{k(j)}\ne 0} \frac{(\wh{\mbb B}-\mbb B)_{i,j}^2}{{\bar \rho_{k(i)} \bar \rho_{k(j)}}}}  
\sqrt{\sum_{i,j} \bar \rho_{k(i)} \bar \rho_{k(j)}} 
+ \sum_{i,j: \bar \rho_{k(i)}= 0 \tx{ or } \bar \rho_{k(j)}=0} |\mbb B_{i,j}|,\\
\end{align*} 
where $k(i)$ is the bucket that contains word $i$. The first term, $\sqrt{\sum_{i,j: \bar \rho_{k(i)}\ne 0 \tx{ and } \bar \rho_{k(j)}\ne 0} \frac{(\wh{\mbb B}-\mbb B)_{i,j}^2}{{\bar \rho_{k(i)} \bar \rho_{k(j)}}}}$, is equal to the Frobenius norm of matrix $D^{-1} (\wh{\mbb B}-\mbb B) D^{-1}$ which is bounded by $\sqrt{R}$ times the spectral norm of matrix $\wh B' - D^{-1}\mbb B D^{-1}$. Inside the term $\sqrt{\sum_{i,j}\bar \rho_{k(i)} \bar \rho_{k(j)}}$, each $\bar \rho_{k(i)}$ is at most $e$ times the empirical marginal $\wh{\rho_i}$ and hence $\sum_i \bar \rho_{k(i)}$ is at most $e$, which gives us a $e$ upper bound for $\sqrt{\sum_{i,j}\bar \rho_{k(i)} \bar \rho_{k(j)}}$. 

The next term, $\sum_{i,j: \bar \rho_{k(i)}= 0 \tx{ or } \bar \rho_{k(j)}=0} |\mbb B_{i,j}|,$ is slightly more complicated to bound, and we analyze the two type of reasons for $\bar \rho_{k}$ to be $0$:
\begin{enumerate}
\item The bucket $k$ has less than $20 e^{-\frac{3}{2}(k+1)}N$ words. 
In this case, the true probability mass of this bucket will also be very small, specifically less than the mass of good words plus the mass of spillover words which is bounded by $20e^{-\frac{k+1}{2}} + e^{-e^{k-2}}$. Taking a summation over $k$ from $k_0$ to $\log(N)$ yields a $22e^{-\frac{k_0+1}{2}}$ upper bound.
\item Consider merging all buckets with $k<k_0$ into a single big bucket.  The true probability mass of this bucket will be less than the mass of good words, which is bounded by $\frac{Me^{k_0}}{N}$, plus the mass of spillover words, which is bounded by $e^{-e^{k_0-2}}$ with high probability.
\end{enumerate}
Putting these two parts together, we have $\sum_{i,j: \bar \rho_{k(i)}= 0 \tx{ or } \bar \rho_{k(j)}=0} |\mbb B_{i,j}|\le 22e^{-\frac{k_0+1}{2}} + \frac{Me^{k_0}}{N} + e^{-e^{k_0-2}}$.
Now with the help of Corollary~\ref{cor:estimator}, we have established the error bound: 
$$
\|\wh{\mbb B}-\mbb B\|_{\ell_1}  \le C\sqrt{R} (\frac{R\log^2(N) }{\sqrt{w_{min}N}} + (\frac{1}{e^{k_0} w^2_{min}})^{1/4}) + 22e^{-\frac{k_0+1}{2}} + \frac{Me^{k_0}}{N} + e^{-e^{k_0-2}}.
$$
Let $k_0= 4\log(\frac{C\sqrt{R}}{\epsilon \sqrt{w_{min}}})+16 > \log(\log(\frac{2}{w_{min}}))+3$, which implies that $C\sqrt{R} (\frac{1}{e^{k_0} w^2_{min}})^{1/4}+ 22e^{-\frac{k_0+1}{2}} + e^{-e^{k_0-2}}\le \frac{1}{2}\epsilon$. Further given that $N=\frac{4Me^{k_0}}{\epsilon} = O(\frac{M R^2}{w_{min}^2\epsilon^5})$ it follows that $\frac{CR^2\log^2(N) }{\sqrt{w_{min}N} }+ \frac{Me^{k_0}}{N}\le \frac{1}{2}\epsilon$. Hence $\|\wh{\mbb B}-\mbb B\|_{\ell_1}  \le \epsilon$.
\end{proof}

\appendix

\section{Auxiliary Lemmas}

\begin{lemma}[Wedin's theorem applied to rank-1 matrices]
  \label{lem:wedin-rank-1}
  Denote symmetric matrix $X= vv^\top + E $.  Let $\wh v\wh v^\top$ denote the rank-1 truncated SVD
  of $X$. There is a positive universal constant $C$   such that
    $$
    \min \{\| v - \wh v \|, \| v + \wh v \|\} \le \min\{ C{\|E\|}^{1/2}, C\frac{\|E\|}{\|v\|}\}
     $$
     \end{lemma}
\begin{proof}
The proof follows directly from application of Wedin's theorem (see e.g. Theorem 4 in\cite{vu2011singular}).
\end{proof}

\begin{lemma}
  \label{claim:sqrt-perturbation}
  Let $U$ be a matrix of dimension $M\times R$. Let $P$ be a projection matrix, we have
  \begin{align*}
    \|U- PU\|^2 \le { \|UU^\top - PU(PU)^\top \|}.
  \end{align*}
\end{lemma}
\begin{proof}(to Lemma~\ref{claim:sqrt-perturbation} )
  Let $P^\perp = I - P$, so $U-PU = P^\perp U$. We can write
  \begin{align*}
    UU^\top - PU(PU)^\top &= (P+P^\perp)U U^\top (P+P^\perp) - PU(PU)^\top
    \\
    &= P^\perp UU^\top P^\perp + P UU^\top P^\perp + P^\perp UU^\top P.
  \end{align*}
  Let vector $v$ denote the leading left singular vector of $P^\perp U$ and $P^\perp U$, by
  orthogonal projection it must be that $P v = 0$. We can bound
  \begin{align*}
    \| UU^\top - PU(PU)^\top \|&\ge |v^\top ( P^\perp UU^\top P^\perp + P UU^\top P^\perp + P^\perp
    UU^\top P) v |
    \\
    &= |v^\top P^\perp UU^\top P^\perp v|
    \\
    &= \|P^\perp U\|^2.
  \end{align*}

\end{proof}


\begin{proposition}[Scaled noise matrix]
  \label{thm:sub-exponential-facts}
  Consider a noise matrix $E\in R^{M \times M}$ with independent entries, and each entry has zero mean and variance $\sigma^2_{i,j}\leq \sigma^2$. 
  Consider a fixed matrix $V$ of dimension $M\times R$ whose columns are orthonormal, with large probability we can bound the norm of $V^\top E V$ and $V^\top E$ separately by:
  \begin{align*}
    \|V^\top E_S V\| = O( {R \sigma})
  \end{align*}
\end{proposition}

\begin{proof}
  To bound the norm of the projected matrix, note that we have
  \begin{align*}
    \|V^\top E_S V \|_2^2 \le \|V^\top E_S V \|_F^2 = Tr(V^\top E_S VV^\top \wt
    E^\top V).
  \end{align*}
  By Markov inequality, we have
  \begin{align*}
    \Pr(Tr(\wh V^\top E_S VV^\top E_S^\top \wh V) > t )\le {1\over t} \mbb E Tr(\wh
    V^\top E_S VV^\top E_S^\top \wh V) = {1\over t}Tr( \wh V^\top \underbrace{\mbb
      E[E_S VV^\top E_S^\top ] }_X\wh V ) = {1\over t} {R^2  \sigma^2},
  \end{align*}
  where the last equality is because for the $i,j$-th entry of $X$ (let $E_i$ denote the
  $i$-th row of $E$ and $V_{r}$ denote the $r$-th column of $V$)
  \begin{align*}
    X_{i,j} = \mbb E[\sum_r (E_i V_r) (E_j V_r) ] = \delta_{i,j} \sum_{r,k} \sigma_{i,k}^2V_{k,r}^2  \le \delta_{i,j} { R\sigma^2}.
  \end{align*}
  Therefore, with probability at least $1- \delta$, we have
  \begin{align*}
    \|V^\top E_S V \| \le \sqrt{R^2 \sigma^2\over \delta}.
  \end{align*}

\end{proof}

\begin{proposition}[Chernoff Bound for Poisson Random Variables (Theorem 5.4 in~\cite{mitzenmacher2005probability})]
  \label{prop:chernoff-poisson}
  \begin{align*}
    &\Pr(\tx{Poi}(\lambda)\ge x)\le {e^{-\lambda}\lt({e\lambda \over x}\rt)^{x}}, \quad
    \tx{for } x>\lambda,
    \\
    &\Pr(\tx{Poi}(\lambda)\le x)\le {e^{-\lambda}\lt({e\lambda \over x}\rt)^{x}}, \quad
    \tx{for } x<\lambda.
  \end{align*}
\end{proposition}

\begin{lemma}[{Variant of Theorem 2.1 in~\cite{le2015concentration} Adapted to Poisson instead of Bernoulli Random Variables}]
  \label{lem:vershynin}

  Consider a random matrix $A$ of size $M\times M$, where each entry follows an independent Poisson
  distribution $A_{i,j}\sim \tx{Poi}(P_{i,j})$. Define $d_{\max} = M\max_{i,j}P_{i,j}$.
  For any $r\ge 1$, the following holds with probability at least $1-M^{-r}$.
  Consider any subset consisting of at most $10{M\over d_{\max}}$, and {decrease the entries} in  the
  rows and the columns corresponding to the indices in the subset in an arbitrary way.
  Then for some universal constant $C$ the modified matrix $A'$ satisfies:
  \begin{align*}
    \|A' - \mbb E A\| \le C r^{3/2} (\sqrt{d_{\max}} + \sqrt{d'}),
  \end{align*}
  where $d'$ denote the maximal row sum in the modified random matrix.
\end{lemma}
\begin{proof}
  The original proof in \cite{le2015concentration} is for independent Bernoulli entries
  $A_{i,j}\sim\tx{Ber}(P_{i,j})$. However the specific property of the distribution is only used in the proof of Lemma 3.3 and several applications of the Chernoff bound. While the applications of Chernoff bound still hold when the Bernoulli random variables are replaced with Poissons, we will provide the replacement of the Bernstein inequality(i.e. Equation 3.5) as follows:
  
  Recall that a random variable $X$ is sub-exponential if there are non-negative parameters
  $(\sigma, b)$ such that $\mbb E[e^{t(X-\mbb E[X])}]\le e^{t^2\sigma^2/2}$ for all $|t|<{1\over b}$.
  Note that a Poisson variables $X\sim\tx{Poi}(\lambda)$ has sub-exponential tail bound with parameters
  $(\sigma = \sqrt{2\lambda}, b=1)$, since
  \begin{align*}
    \log( \mbb E[e^{t(X-\lambda)} ] e^{-t^2\sigma^2/2 }) = (\lambda(e^t-1) - \lambda t) - \lambda
    t^2 \le 0, \tx{ for } |t|<1.
  \end{align*}
  Notice that both the centered random variable $X-\lambda$ and the flipped random variable $-(X-\lambda)$ are sub-exponential with the same parameters as $X$. Therefore, when the entries are replaced by independent Poisson entries
  $A_{i,j}\sim\tx{Poi}(P_{i,j})$, we can apply Bernstein inequality for sub-exponential random variables to yield a similar concentration bound as Equation 3.5:
  \begin{align*}
    \Pr(|X_{i}| > tm) \le 2 \exp({ -m t^2/2 \over  d_{max}/n + b t })\le 2 \exp({ -m t^2/2 \over  d_{max}/n + t }).
  \end{align*}
  The same arguments of the proof in \cite{le2015concentration} then go through.
\end{proof}

\newcommand{\Var}{\mathbf{Var}}
\def\A{{\bf A}}
\def\a{{\bf a}}
\def\B{{\bf B}}
\def\b{{\bf b}}
\def\C{{\bf C}}
\def\c{{\bf c}}
\def\D{{\bf D}}
\def\d{{\bf d}}
\def\E{{\bf E}}
\def\e{{\bf e}}
\def\f{{\bf f}}
\def\F{{\bf F}}
\def\K{{\bf K}}
\def\k{{\bf k}}
\def\L{{\bf L}}
\def\H{{\bf H}}
\def\G{{\bf G}}
\def\I{{\bf I}}
\def\i{{\bf i}}
\def\j{{\bf j}}
\def\R{{\bf R}}
\def\X{{\bf X}}
\def\Y{{\bf Y}}
\def\P{{\bf P}}
\def\Q{{\bf Q}}
\def\s{{\bf s}}
\def\S{{\bf S}}
\def\T{{\bf T}}
\def\x{{\bf x}}
\def\y{{\bf y}}
\def\z{{\bf z}}
\def\Z{{\bf Z}}
\def\M{{\bf M}}
\def\m{{\bf m}}
\def\n{{\bf n}}
\def\U{{\bf U}}
\def\u{{\bf u}}
\def\V{{\bf V}}
\def\v{{\bf v}}
\def\W{{\bf W}}
\def\w{{\bf w}}
\def\0{{\bf 0}}
\def\1{{\bf 1}}
\def\tX{{\widetilde{\X}}}

\def\AM{{\mathcal A}}
\def\EM{{\mathcal E}}
\def\FM{{\mathcal F}}
\def\TM{{\mathcal T}}
\def\UM{{\mathcal U}}
\def\XM{{\mathcal X}}
\def\YM{{\mathcal Y}}
\def\NM{{\mathcal N}}
\def\OM{{\mathcal O}}
\def\QM{{\mathcal Q}}
\def\IM{{\mathcal I}}
\def\GM{{\mathcal G}}
\def\PM{{\mathcal P}}
\def\LM{{\mathcal L}}
\def\MM{{\mathcal M}}
\def\DM{{\mathcal D}}
\def\SM{{\mathcal S}}
\def\RB{{\mathbb R}}
\def\EB{{\mathbb E}}

\def\ty{\tilde{\bf y}}
\def\tz{\tilde{\bf z}}
\def\hd{\hat{d}}
\def\HD{\hat{\bf D}}
\def\hx{\hat{\bf x}}
\def\hR{\hat{R}}

\def\ph{\mbox{\boldmath$\phi$\unboldmath}}
\def\Pii{\mbox{\boldmath$\Pi$\unboldmath}}
\def\pii{\mbox{\boldmath$\pi$\unboldmath}}
\def\Ph{\mbox{\boldmath$\Phi$\unboldmath}}
\def\Ps{\mbox{\boldmath$\Psi$\unboldmath}}
\def\tha{\mbox{\boldmath$\theta$\unboldmath}}
\def\muu{\mbox{\boldmath$\mu$\unboldmath}}
\def\Si{\mbox{\boldmath$\Sigma$\unboldmath}}
\def\Gam{\mbox{\boldmath$\Gamma$\unboldmath}}
\def\Lam{\mbox{\boldmath$\Lambda$\unboldmath}}
\def\De{\mbox{\boldmath$\Delta$\unboldmath}}
\def\vps{\mbox{\boldmath$\varepsilon$\unboldmath}}
\def\Up{\mbox{\boldmath$\Upsilon$\unboldmath}}
\def\Lap{\mbox{\boldmath$\LM$\unboldmath}}
\newcommand{\ti}[1]{\tilde{#1}}

\def\tr{\mathrm{tr}}
\def\etr{\mathrm{etr}}
\def\etal{{\em et al.\/}\,}
\newcommand{\indep}{{\;\bot\!\!\!\!\!\!\bot\;}}
\def\argmax{\mathop{\rm argmax}}
\def\argmin{\mathop{\rm argmin}}
\def\cov{\text{cov}}
\def\dg{\text{diag}}

\section{Sample Complexity Lowerbound for $2$-State HMM} \label{sec:hmmlb}

We establish the following theorem for \emph{testing} whether a sequence of observations are drawn from a 2 state HMM, versus are i.i.d. sampled from $\mc M$.

\medskip
\noindent \textbf{Theorem~\ref{thm:testing-hmm}.} \emph{There exists a constant $c>0$ such that for sufficiently large $M$, given a sequence of observations from a HMM with two states and
emission distributions $p,q$ supported on $M$ elements, even if the
underlying Markov process is symmetric, with transition probability
$1/4$, it is information theoretically impossible to distinguish the
case that the two emission distributions, $p=q=\textrm{Unif}[M]$ from the case
that $||p-q||_1 =1$ with probability greater than $2/3$ using a sequence of fewer than $cM$
observations.
}
\medskip

We first define the family of two state HMMs to which our lower bound will apply. 
Define a distribution $\mathcal{D}_n$ over 2-state HMMs as follows:  the underlying Markov process is symmetric, with two states ``$+$'' and ``$-$'', with probability of changing state equal to $1/4$.  The distribution of observations given state ``$+$'' is uniform over a uniformly random subset $S_1 \subset \{1,\ldots, n\}$ with $|S_1| = n/2$, and the distribution of observations given state ``$-$'' is the uniform distribution over set $S_2 = \{1,\ldots,n\} \setminus S_1.$

\begin{proposition}\label{thm:distinguish}
No algorithm can distinguish a length $c n$ sequence of observations drawn from a 2-state HMM drawn according to $\mathcal{D}_n,$ from a uniformly random sequence of $cn$ independent draws from $\{1,\ldots,n\}$ with probability of success greater than $\frac{1}{2} + O\left(\sqrt{\sqrt{\frac{2}{2-\frac{4c}{3}}}-1}\right)$.   Specifically, the distribution of random sequences of $cn$ draws from $\{1,\ldots,n\}$ has total variation distance at most $\frac{1}{2}\sqrt{\sqrt{\frac{2}{2-\frac{4c}{3}}}-1} + o_n(1)$ from the distribution of sequences of $cn$ observations drawn from a 2-state HMM drawn according to $\mathcal{D}_n$.
\end{proposition}
To establish the above theorem, it will be convenient to consider a \emph{labelled} sequence of observations, where the label of the $i$th element, $\sigma_i$, corresponds to the hidden state $+$ or $-$.  In some sense, the high level idea of the proof is to argue that given a uniformly random sequence of draws from $\{1,\ldots,n\},$ it is possible to assign labels to the observations, such that the labelled sequence is information theoretically indistinguishable to a labelled sequence generated from the 2-state HMM.  Now we define the joint distribution of observations and labels in our two state HMM.

\begin{definition}
Define $G\in \{1,\ldots,n\}^k$ as a length $k$ sequence. Let $P_n(G,\sigma)$ be the joint distribution of a labelled length $k$ sequence output by a uniformly random two state HMM drawn according $\mathcal{D}_n$, then 
$$
P_n(G,\sigma) = P_n(G|\sigma)P_n(\sigma) =  \frac{\prod_{i=1}^{k-1} (\I\{\sigma_{i}=\sigma_{i+1}\}\frac{3}{2}+ \I\{\sigma_{i}\neq\sigma_{i+1}\}\frac{1}{2})}{n^{k}}\frac{1}{\binom{n}{n/2}}
$$
\end{definition}
The following defines the distribution of the uniform model's observations.
\begin{definition}
Let $P'_n(G)$ be the distribution of a uniformly random sequence of $k$ independent draws from $\{1,\ldots,n\}$, then for any $G$, 
$$
P'_n(G) = \frac{1}{n^{k}}.
$$ 
We now define a distribution over labelings which will allow us to assign a labelling to a uniformly random sequence (corresponding to $P'_n$):
$$
P'_n(\sigma|G)=\frac{P_n(G|\sigma)}{\sum_{\sigma}P_n(G|\sigma)}
$$ 
\end{definition}
\begin{definition}
Define the random variable $Y_n = \frac{P_n(G,\sigma)}{P'_n(G,\sigma)} = \frac{n^k\sum_{\sigma}P_n(G|\sigma)}{\binom{n}{n/2}}$, and observe that $\frac{1}{2}\E_{P'_n}|Y_n-1|=D_{TV}(P_n,P'_n)$. 
\end{definition}

Our proof approach will be to explicitly bound the variance of $Y_n$, which will immediately yield Theorem~\ref{thm:distinguish} via the following trivial lemma:

\begin{lemma}\label{lm:absvariance}
If $\Var[Y_n]\leq \epsilon$, then $\E_{P'_n}|Y_n-1|\leq \sqrt{\epsilon}$
\end{lemma}
\begin{proof}
Let $X = |Y_n-1|$ and note that $\E_{P'_n}[Y_n]=1$.  Hence $\E_{P'_n}[|Y_n - 1|]^2 = \E_{P'_n}[X]^2\leq \E_{P'_n}[X^2]= \Var[Y_n] \le \epsilon.$
\end{proof}

For notational convenience, we will write $\E$ instead of $\E_{P'_n}$ for the remainder of the proof. To bound $\Var[Y_n],$ it will be convenient to have a relatively clean expression for each of the ``cross-terms'' in the variance calculation.  The following lemma establishes such an expression, in terms of the overlap between the  labelings corresponding to the two components of each cross-term. 

\begin{lemma}
Given $\sigma,\pi \in \{+,-\}^n$, assume $|\sigma^+\cap \pi^+|= a$, i.e. there are $a$ symbols that have ``+'' label under both $\sigma$ and $\pi$. Then $E[P_n(G|\sigma)P_n(G|\pi)] = r(\frac{2a}{n}),$ where 
\begin{align*}
r(p) = \frac{2^{-3 k-1}n^{-2k}}{{\sqrt{64 (p-1) p+25}}}\\
 [\sqrt{64 (p-1) p+25} \left(\left(5+\sqrt{64 (p-1) p+25}\right)^{k}+\left(5-\sqrt{64 (p-1) p+25}\right)^{k}\right)\\
   +3
   \left(\left(5+\sqrt{64 (p-1) p+25}\right)^{k}-\left(5-\sqrt{64 (p-1) p+25}\right)^{k}\right)
]
\end{align*}

\end{lemma}
\begin{proof}
Let $G^{(k)}$ denote a sequence of length $k$. Define the following quantity:
\begin{align}
F_{t,++} = \E[P_n(G^{(t)}|\sigma)P_n(G|\pi)|G_t\in(\sigma^+\cap\pi^+)]\\
F_{t,+-} = \E[P_n(G^{(t)}|\sigma)P_n(G|\pi)|G_t\in(\sigma^+\cap\pi^-)]\\
F_{t,-+} = \E[P_n(G^{(t)}|\sigma)P_n(G|\pi)|G_t\in(\sigma^-\cap\pi^+)]\\
F_{t,--} = \E[P_n(G^{(t)}|\sigma)P_n(G|\pi)|G_t\in(\sigma^-\cap\pi^-)]
\end{align}
Let $p = 2a/n$. There is a simple linear recurrence formula: 
$$
\begin{pmatrix}
F_{t+1,++}\\
F_{t+1,+-}\\
F_{t+1,-+}\\
F_{t+1,--}
\end{pmatrix} = \frac{1}{8n^2}\begin{pmatrix}
9p&3(1-p)&3(1-p)&p\\
3p&9(1-p)&(1-p)&3p\\
3p&(1-p)&9(1-p)&3p\\
p&3(1-p)&3(1-p)&9p
\end{pmatrix}
\begin{pmatrix}
F_{t,++}\\
F_{t,+-}\\
F_{t,-+}\\
F_{t,--}
\end{pmatrix}
$$
Finally, 
\begin{align}
E[P_n(G|\sigma)P_n(G|\pi)] = \frac{1}{2}(F_{k,++}p+F_{k,+-}(1-p)+F_{k,-+}(1-p)+F_{k,--}p)\\
= \frac{1}{2^{3k-2}n^{2k}}\begin{pmatrix}
p&1-p&1-p&p
\end{pmatrix}
\begin{pmatrix}
9p&3(1-p)&3(1-p)&p\\
3p&9(1-p)&(1-p)&3p\\
3p&(1-p)&9(1-p)&3p\\
p&3(1-p)&3(1-p)&9p
\end{pmatrix}^{k-1}
\begin{pmatrix}
1\\
1\\
1\\
1
\end{pmatrix}.
\end{align}
Simplifying the above product yields the claimed expression:
\begin{align*}
r(p) = \frac{2^{-3 k-1}n^{-2k}}{{\sqrt{64 (p-1) p+25}}}\\
 [\sqrt{64 (p-1) p+25} \left(\left(5+\sqrt{64 (p-1) p+25}\right)^{k}+\left(5-\sqrt{64 (p-1) p+25}\right)^{k}\right)\\
   +3
   \left(\left(5+\sqrt{64 (p-1) p+25}\right)^{k}-\left(5-\sqrt{64 (p-1) p+25}\right)^{k}\right)
]
\end{align*}
\end{proof}

\begin{proposition}\label{thm:variance}
Let $k=cn$ for some non-negative constant $c<3/2$.

$$
\E[Y_n^2] \leq \sqrt{\frac{2}{2-\frac{4c}{3}}}+o(1)
$$
\end{proposition}
\begin{proof}
$$
\E[Y_n^2] = \frac{n^{2cn}\sum_{\sigma,\pi}P_n(G|\sigma)P_n(G|\pi)}{\binom{n}{n/2}^2} = \frac{n^{2cn}}{\binom{n}{n/2}^2}\sum_{\sigma,\pi}E[P_n(G|\sigma)P_n(G|\pi)]
$$

There are $\binom{n}{n/2}$ different $\pi$, for each of them there exists $\binom{a}{n/2}^2$ different $\sigma$ such that $|\pi^+\cap\sigma^+| = a$. Hence the above formula equals
$$
= \frac{\binom{n}{n/2}\sum_{a=0}^{n/2} \binom{a}{n/2}^2n^{2cn}r(2a/n)}{\binom{n}{n/2}^2} = \frac{\sum_{a=0}^{n/2} \binom{a}{n/2}^2n^{2cn}r(2a/n)}{\binom{n}{n/2}}.
$$
We will use Stiring's approximation to simplify this expression, though we first show that the contribution from the first or last $O(\log n)$ terms, when $a$ or $n/2-a$ is small, is negligible. 
\begin{align}
\frac{\sum_{a=0}^{\log(n)} \binom{n/2}{a}^2r(2a/n)}{\binom{n}{n/2}} = O(\frac{\log(n) (n/2)^{2\log(n)}\frac{5}{4}^{cn}\sqrt{n}}{2^n})=o(1)\label{eq:forst}
\end{align}
It's not hard to verify the first inequality by plugging in $a=\log(n)$ and applying Stirling's approximation on $\binom{n}{n/2}$, and since $\frac{5}{4}^c<2$ we have the last equality.
Applying Stiring's approximation to $\binom{n/2}{a}^2$ for each $\log(n)<a<n/2-\log(n)$ yields
\begin{align}
\frac{\sum_{a=0}^{n/2} \binom{n/2}{a}^2n^{2cn}r(2a/n)}{\binom{n}{n/2}} = \frac{\sum_{a=\log(n)}^{n/2-\log(n)} \binom{n/2}{a}^2r(2a/n)}{\binom{n}{n/2}}+o(1) \\
=\left(\sum_{a=\log(n),\log(n)+1,...n/2-\log(n)}\frac{2^{-(3 c+1) n-\frac{3}{2}} (1-\frac{2a}{n})^{2a-n-1} \frac{2a}{n}^{-2a-1} }{\sqrt{\pi } \sqrt{n} \sqrt{64 (\frac{2a}{n}-1) \frac{2a}{n}+25}} s(\frac{2a}{n})\right)(1+O(\frac{1}{\log n}))+o(1)\\
=\left(\sum_{a=0,1,...n/2}\frac{2^{-(3 c+1) n-\frac{3}{2}} (1-\frac{2a}{n})^{2a-n-1} \frac{2a}{n}^{-2a-1} }{\sqrt{\pi } \sqrt{n} \sqrt{64 (\frac{2a}{n}-1) \frac{2a}{n}+25}} s(\frac{2a}{n})\right)+o(1) \label{eq10}
\end{align}
Where 
\begin{align*}
s(p)=
\sqrt{64 (p-1) p+25} \left(\left(5+\sqrt{64 (p-1) p+25}\right)^{c n}+\left(5-\sqrt{64 (p-1) p+25}\right)^{c n}\right)\\
   +3
   \left(\left(5+\sqrt{64 (p-1) p+25}\right)^{cn}-\left(5-\sqrt{64 (p-1) p+25}\right)^{c n}\right)
\end{align*}
The last equality holds since the first and last $\log(n)$ terms' constribution to the sum is $o(1)$.
Note that we can drop all $5-\sqrt{64 (p-1) p+25}$ terms and only incur $O(\frac{1}{2^n})$ multiplicative error. Since $3$ is always smaller than $\sqrt{64 (p-1) p+25}$, ignoring all $o(1)$ terms yields that the expression in Equation~\ref{eq10} is bounded by the following:
\begin{align}
\leq \sum_{a=0,1,...n/2}\frac{2^{-(3 c+1) n-\frac{3}{2}} (1-\frac{2a}{n})^{2a-n-1} \frac{2a}{n}^{-2a-1}}{\sqrt{\pi } \sqrt{n}} \left(\sqrt{64 (\frac{2a}{n}-1) \frac{2a}{n}+25}+5\right)^{cn}.
\end{align}
Replacing $a$ with $\frac{n}{2}(1/2-\epsilon)$ and grouping the constants together yields
\begin{align}
=\left(\sum_{\epsilon=-\frac{1}{2},-\frac{1}{2}+\frac{2}{n},...\frac{1}{2}}\frac{2^{3/2}}{\sqrt{\pi}\sqrt{n}} \left((\frac{5+\sqrt{9+64\epsilon^2}}{8})^c(1+2\epsilon)^{-\frac{1}{2}-\epsilon}(1-2\epsilon)^{-\frac{1}{2}+\epsilon}\right)^n \right) \label{eq12}
\end{align}
Note that $\frac{5+\sqrt{9+64\epsilon^2}}{8}\leq e^{\frac{4}{3}\epsilon^2}$ and $(1+2\epsilon)^{-\frac{1}{2}-\epsilon}(1-2\epsilon)^{-\frac{1}{2}+\epsilon}\leq e^{-2\epsilon^2}$ when $-1/2<\epsilon<1/2$, hence the summand in Equation~\ref{eq12} is bounded by:
$$
\frac{2^{3/2}}{\sqrt{\pi}\sqrt{n}} \left((\frac{5+\sqrt{9+64\epsilon^2}}{8})^c(1+2\epsilon)^{-\frac{1}{2}-\epsilon}(1-2\epsilon)^{-\frac{1}{2}+\epsilon}\right)^n \leq \frac{2^{3/2}}{\sqrt{\pi}\sqrt{n}} e^{(\frac{4c}{3}-2)\epsilon^2n}.
$$

Letting $n$ go to $\infty$, and converting the sum into integral yields the following asymptotic bound
\begin{align}
\lim_{n\to \infty}\sum_{\epsilon=0,\frac{2}{n},...\frac{1}{2}}\frac{2^{5/2}}{\sqrt{\pi}\sqrt{n}} e^{(\frac{4c}{3}-2)\epsilon^2n} = \lim_{n\to \infty}\frac{2^{5/2}}{\sqrt{\pi}}\sum_{x=0,\frac{1}{\sqrt{n}},...\frac{1}{4}\sqrt{n}} e^{4(\frac{4c}{3}-2)x^2}\frac{1}{\sqrt{n}}\\
 = \frac{2^{5/2}}{\sqrt{\pi}} \int_{x=0}^{\infty} e^{4(\frac{4c}{3}-2)x^2}\\
 = \frac{2^{5/2}}{\sqrt{\pi}} \frac{\sqrt{\pi}}{4\sqrt{2-\frac{4c}{3}}}= \sqrt{\frac{2}{2-\frac{4c}{3}}}
\end{align}
\end{proof}
\begin{proof}[Proof of Theorem~\ref{thm:distinguish}]
By Proposition~\ref{thm:variance} and the fact that $\E[Y_n]=1$, we have $\Var[Y_n]\leq \sqrt{\frac{2}{2-\frac{4c}{3}}}+o(1)$. Hence by Lemma~\ref{lm:absvariance}, the total variation distance $D_{TV}(P_n,P'_n)\leq \frac{1}{2}\sqrt{\sqrt{\frac{2}{2-\frac{4c}{3}}}-1}+o(1)$.
\end{proof}

\section{Proof of Proposition~\ref{prop:hmm} }

In this section we prove Proposition~\ref{prop:hmm}, restated below for convenience:

\medskip
\noindent \textbf{Proposition~\ref{prop:hmm}.} \emph{
(Learning 2-state HMMs)   Consider a sequence of observations given by a Hidden Markov Model with two hidden states and symmetric transition matrix with entries bounded away from 0.  Assuming a constant $\ell_1$ distance between the  distributions of observations corresponding to the two states, there exists an algorithm which,
  given a sampled chain of length $N=\Omega(M/\eps^2)$, runs in time
  $\textrm{poly}(M)$ and returns estimates of the transition matrix and two observation distributions that is accurate in $\ell_1$ distance, with probability at least $2/3$.}
  \medskip

Consider the expected bigram matrix corresponding to the 2 state HMM: $\mbb B = P W P^T$ where $W=\begin{pmatrix} 
1-t & t \\
t & t 
\end{pmatrix}$.  Letting $\rho$ denote the vector of marginal probabilities of each of the $M$ words, and note that $\mbb B - \rho \rho^T$ is a rank 1 symmetric (and p.s.d.) matrix, whose rank 1 factor is a multiple of the vector of differences between the two observation distributions: 
$$
\mbb B - \rho \rho^T= (\frac{1}{2}-t)(p_1-p_2)(p_1-p_2)^T
$$
Observe that when $t>1/2$, $W$ is not psd and hence we can't apply the rank $R$ algorithm directly. But since $(\frac{1}{2}-t)(p_1-p_2)(p_1-p_2)^T$ is either psd or nsd, we can slightly modify the rank $R$ algorithm to estimate $\sqrt{|\frac{1}{2}-t|}(p_1-p_2)$. Specifically, in the Step 2(b), instead of computing singular vectors of $\wt B_k$, we compute the top singular vector of $\wt B_k -\wh \rho_{\mc I_k} \wh \rho_{\mc I_k}^\top$. Then in Step 3, instead of computing SVD of $P_V D^{-1} B D^{-1} P_V$, we use $P_V D^{-1} (B-\wh\rho\wh\rho^\top) D^{-1} P_V$ instead. The result would be an accurate estimate of $(\frac{1}{2}-t)(p_1-p_2)(p_1-p_2)^\top$ in $\ell_1$ distance. This would be sufficient to learn the probability of each word in the two observation distributions, provided we know the transition probability $t$.  

We now argue how to accurately recover $p$.  Define $S$ to be a set of words that all have positive value in $p_1-p_2$. In aggregate, set $S$ has significantly different probability under the two observation distributions---namely differing by a constant.  Note that by assumption on the $\ell_1$ distance between the two distribution, such a set exists. 
 We will require the use of tri-gram statistics, but only those statistics corresponding to an HMM with output alphabet consisting of two ``super words'', with one word corresponding to the entire set $S$, and the other corresponding to the complement of $S$.  Construct and solve a single cubic equation with 1 variable whose coefficients are estimated from tri-gram, to determine the transition probability $t$.  The following lemma shows that the modified Step 2 share the same property as our rank R algorithm.

\begin{lemma}[{Estimating the separation vector restricted to bins}]
  \label{lem:learnDeltak}
Let $\wt B_k$ be the $k$'th block of the regularized matrix in Step 2 of our algorithm. $\wh \rho_{\mc I_k}$ be the estimated marginal restricted to bucket $\mc I_k$. With probability $1-o(1)$, we have
$$
\lt\| (\wt B_k- \wh\rho_{{\mc I}_k}\wh\rho_{{\mc I}_k}^\top)
      -( {\mbb B}_k - \rho_{\mc I_k}\rho_{\mc I_k}^\top)\rt\|  
            = O\lt(\sqrt{|\mc I_k| \bar \rho_k^2 \over N w_{min}}\rt) 
$$
\end{lemma}
\begin{proof}
  Recall the result of Lemma~\ref{lem:diag-block-concen} concerning the concentration of the diagonal
  block after regularization. For the $k$'th empirical bin, with high probability, $\|\wt B_{\mc I_k\times \mc I_k}-\mbb B_{\mc I_k\times \mc I_k}\| = O\lt(\sqrt{|\mc I_k| \bar \rho_k^2 \over N w_{min}} \rt).$ Recall that $\wh \rho_{{\mc I}_k}$ is defined to be the empirical marginal vector $\wh \rho$ restricted to the empirical bin ${\mc I}_k$.
  We can bound 
    $$
      \lt\| (\wt B_k- \wh\rho_{{\mc I}_k}\wh\rho_{{\mc I}_k}^\top)
      -( {\mbb B}_k - \rho_{\mc I_k}\rho_{\mc I_k}^\top)\rt\|  
      \le 
      \lt\|  \wt B_k - {\mbb B}_k \rt\|
      + \lt\| \wh\rho_{{\mc I}_k}\wh\rho_{{\mc I}_k}^\top - \rho_{\mc I_k}\rho_{\mc I_k}^\top\rt\|.
     $$
The second term satisfies the following inequality:$ \lt\| \wh\rho_{{\mc I}_k}\wh\rho_{{\mc I}_k}^\top - \rho_{\mc I_k}\rho_{\mc I_k}^\top\rt\|  \le   \lt\|\wh\rho_{{\mc I}_k} - \rho_{\mc I_k} \rt\| (\lt\|\wh\rho_{{\mc I}_k}\rt\|+\lt\|\rho_{{\mc I}_k}\rt\|)$. In order to bound the second term $ \|\wh\rho_{{\mc I}_k}- \rho_{\mc I_k}\|_2$, we first establish a loose bound of $\|\wh\rho- \rho\|_2$. By Markov's inequality, $P(\sum_{i=1}^M (\wh\rho_i - \rho_i)^2> \frac{\delta}{N})\le \frac{\E[\sum_{i=1}^M (\wh\rho_i - \rho_i)^2]}{\delta/N}= \frac{1}{\delta} $. Hence with large probability, for all $k$, $\|\wh\rho_{{\mc I}_k}- \rho_{\mc I_k}\|_2 \le O(\frac{1}{\sqrt{N}})$.  Further, we have $\|\rho_{{\mc I}_k}\|\le \sqrt{\bar \rho^2_k |\mc I_k|+\frac{e^{-e^{k-2}}}{N}}\le 2\sqrt{\bar \rho^2_k |\mc I_k|}$ and $\|\wh \rho_{{\mc I}_k}\|\le \sqrt{\bar \rho^2_k |\mc I_k|}$. Hence we establish an upper bound for the second term: $ \lt\| \wh\rho_{{\mc I}_k}\wh\rho_{{\mc I}_k}^\top - \rho_{\mc I_k}\rho_{\mc I_k}^\top\rt\| \le O(\sqrt{\frac{\bar \rho^2_k |\mc I_k|}{N}})$. The first term is bounded with direct application of Lemma~\ref{lem:diag-block-concen}. Combining the two parts yields:
$$
\lt\| (\wt B_k- \wh\rho_{{\mc I}_k}\wh\rho_{{\mc I}_k}^\top)
      -( {\mbb B}_k - \rho_{\mc I_k}\rho_{\mc I_k}^\top)\rt\|  
            = O\lt(\sqrt{|\mc I_k| \bar \rho_k^2 \over N w_{min}}\rt). 
$$
\end{proof}
\begin{corollary}\label{cor:2shmm_diag-block-concen}
Let $v_k$ be the leading singular vector of regularized block $\wt B_{\mc I_k\times \mc I_k}-\wh\rho_{{\mc I}_k} \wh\rho_{{\mc I}_k}^\top$, Define $P_{V_{k}} = v_k v_k^\top$. Then with high probability, we have
  \begin{align}
    \label{eq:bin-concen-sqrt}
    \| P_{ V_{k}} \delta_{k} - \delta_{k}\| =
    O\lt(\lt(|\mc I_k| \bar \rho_k^2 \over N w_{min} \rt)^{1/4}\rt).
  \end{align}
\end{corollary}
\begin{proof}
The proof is the same as Corollary~\ref{cor:diag-block-concen}.
\end{proof}

\begin{proposition}[Noise Filter]\label{prop:2shmm_noisefilter}
With probability larger than $\frac{3}{4}$, $\|P_V D^{-1} ((B-\wh \rho \wh\rho^\top)- (\mbb B - \rho \rho^\top)) D^{-1} P_V\| \le O(\frac{\log^2(N) }{\sqrt{w_{min}N}}+\sqrt{\frac{M}{N}}+e^{\frac{1}{2}(-e^{k_0-2}-k_0)})$
\end{proposition}
\begin{proof}
The inequality that $\|P_V D^{-1} ((B-\wh \rho \wh\rho^\top)- (\mbb B - \wh \rho \wh \rho^\top)) D^{-1} P_V\| \le O(\frac{\log^2(N) }{\sqrt{w_{min}N}})$ is exactly what Corollary~\ref{prop:noisefilter} shows. Hence we only need to upper bound  $\|P_V D^{-1} ( \rho \rho^\top - \wh \rho \wh \rho^\top) D^{-1} P_V\|$ which is equivalent to $\|D^{-1} ( \rho \rho^\top - \wh \rho \wh \rho^\top) D^{-1}\|$ because $P_V$ is an orthogonal matrix. With the following inequality: $ \lt\| D^{-1}\wh\rho\wh\rho^\top D^{-1} - D^{-1} \rho \rho^\top D^{-1}\rt\|  \le   \lt \| D^{-1}(\wh\rho - \rho )\rt\| (\lt \|D^{-1}\wh\rho\rt\|+\lt\| D^{-1}\rho\rt\|)$, we are going to bound the three terms one by one: $\lt \|D^{-1}\wh\rho\rt\|$, $\lt\| D^{-1}\rho\rt\|$, $\lt \| D^{-1}(\wh\rho - \rho )\rt\|$.  
\begin{enumerate}
\item Expanding $\lt \|D^{-1}\wh\rho\rt\|$ and applying the crude upper-bound $\|\wh \rho_{{\mc I}_k}\|\le \sqrt{\bar \rho^2_k |\mc I_k|}$, we have $\lt \|D^{-1}\wh\rho\rt\|=\sqrt{\sum_k \frac{\|\wh \rho_{\mc I_k}\|^2}{\bar\rho_k}} \le \sqrt{\sum_k \bar\rho_k |\mc I_k|}\le e$.  
\item Applying similar inequality, $\|\rho_{{\mc I}_k}\|\le \sqrt{\bar \rho^2_k |\mc I_k|+\frac{e^{-e^{k-2}}}{N}}\le 2\sqrt{\bar \rho^2_k |\mc I_k|}$, with the same argument, similar bound holds: $\lt \|D^{-1}\rho\rt\|\le 2e$. 
\item The last term $\lt \| D^{-1}(\wh\rho - \rho )\rt\|$ is slightly trickier. Notice that if we replace $D^{-1}$ by matrix $\diag({\rho})^{-1/2}$ whose $i$'th diagonal entry is $\frac{1}{\sqrt{\rho_i}}$, $e \lt \| \diag(\rho)^{-1/2}(\wh\rho - \rho )\rt\|$ would be a good upper bound if we ignore the spillover words. By Markov's inequality, $P(\sum_{i=1}^M \frac{(\wh\rho_i - \rho_i)^2}{\rho_i}> \delta)\le \frac{\E[\sum_{i=1}^M \frac{(\wh\rho_i - \rho_i)^2}{\rho_i}]}{\delta}= \frac{1}{\delta} \frac{M}{N}$. Hence with large probability, $\lt \| \diag(\rho)^{-1/2}(\wh\rho - \rho )\rt\|_2^2=O(\frac{M}{N})$. Now we need  to incorporate the spillover words. For the $k$th bucket, the total contribution of the spillover words to the term $\lt \| D^{-1}(\wh\rho - \rho )\rt\|_2^2$ is exactly  $\frac{\sum_{i\in \mc J_k} (\wh\rho_i - \rho_i )^2}{\bar \rho_k}$ which is smaller than $4\frac{\sum_{i\in \mc J_k} \rho_i ^2}{\bar \rho_k}$, and by Proposition~\ref{prop:small-spillover} smaller than $4e^{-e^{k-2}-k-1}$. Take the sum from $k_0$ to $\log(N)$,$\lt \| D^{-1}(\wh\rho - \rho )\rt\|=O(\sqrt{\frac{M}{N})+e^{-e^{k_0-2}-k_0}})=O(\sqrt{\frac{M}{N}}+e^{\frac{1}{2}(-e^{k_0-2}-k_0)})$. 
\end{enumerate}
To conclude, we have shown that with probability at least $3/4$, $\|P_V D^{-1} ((B-\wh \rho \wh\rho^\top)- (\mbb B -\rho \rho^\top)) D^{-1} P_V\|\le \|P_V D^{-1} ((B-\wh \rho \wh\rho^\top)- (\mbb B - \wh \rho \wh \rho^\top)) D^{-1} P_V\|+\|P_V D^{-1} ( \rho \rho^\top - \wh \rho \wh \rho^\top) D^{-1} P_V\|=O(\frac{\log^2(N) }{\sqrt{w_{min}N}}+\sqrt{\frac{M}{N}}+e^{\frac{1}{2}(-e^{k_0-2}-k_0)})$

\end{proof}

\begin{proposition}[Projection Preserves Signal]\label{prop:2shmm_subspace}
With high probability, $\|P_V D^{-1} (\mbb B-\rho \rho^\top) D^{-1} P_V - D^{-1} (\mbb B-\rho \rho^T) D^{-1}\| \le O((\frac{1}{e^{k_0} w^2_{min}})^{1/4})$
\end{proposition}
\begin{proof}
The proof is similar to the proof of Proposition~\ref{prop:subspace} except $\delta$ plays the role of $\mbb B^{sqrt}$ and we need Corollary~\ref{cor:2shmm_diag-block-concen}.
\end{proof}

\begin{corollary}\label{cor:2shmm_estimator}
Let $\wh B'$ be the rank $1$ truncated SVD of matrix $P_V D^{-1} (B-\wh\rho \wh\rho^\top) D^{-1} P_V$. With probability larger than $\frac{3}{4}$, $\|\wh B' - D^{-1}(\mbb B-\rho\rho^\top) D^{-1}\| =O(\frac{\log^2(N) }{\sqrt{w_{min}N}} + \sqrt{\frac{M}{N}}+e^{\frac{1}{2}(-e^{k_0-2}-k_0)}+(\frac{1}{e^{k_0} w^2_{min}})^{1/4})$
\end{corollary}
\begin{proof}
Follows from the proof of Corollary~\ref{cor:estimator}. Simply combine Proposition~\ref{prop:2shmm_noisefilter} and Proposition~\ref{prop:2shmm_subspace}.
\end{proof}

\begin{proposition}\label{prop:2shmm_delta}
Define $\wh \delta'=\sqrt{\sigma}u$ where $u$ is the left singular vector of $\wh B'$ and $\sigma$ is the singular value. Then $\wh B' = \wh\delta' \wh\delta'^\top$ or $ \wh B' = -\wh\delta' \wh\delta'^\top$ holds. With large probability,  
$$
\min \{\| D^{-1}\delta - \wh \delta' \|, \| D^{-1}\delta + \wh \delta' \|\} \le O(\sqrt{\frac{\log^2(N) }{\sqrt{w_{min}N}} + \sqrt{\frac{M}{N}}+e^{\frac{1}{2}(-e^{k_0-2}-k_0)}+(\frac{1}{e^{k_0} w^2_{min}})^{1/4}})
$$
\end{proposition}
\begin{proof}
Straightforward proof by applying Lemma~\ref{lem:wedin-rank-1} with $\wh \delta'$ and $D^{-1}\delta$.
\end{proof}

\begin{theorem}[Main Theorem]
Let $\wh \delta'$ be the vector defined in Propositon~\ref{prop:2shmm_delta}(flip the sign of $\wh \delta'$ such that $\|D^{-1} \wh \delta' - \delta\|$ achieves the min) and $\wh \delta = D \wh \delta'$. With probability at least $2/3$:
$$
    min\{\|\wh \delta-\delta\|_{\ell_1}, \|\wh \delta-\delta\|_{\ell_1}\}\le \epsilon.
 $$
\end{theorem}
\begin{proof}
Apply Cauchy-Schwatz to the $\ell_1$ norm: 
\begin{align*}
 \|\wh \delta-\delta\|_{\ell_1} &= \sum_{i:\bar\rho_{k(i)}\ne 0} |\wh \delta_i-\delta_i| \frac{1}{\sqrt{\bar \rho_{k(i)}}} \sqrt{\bar \rho_{k(i)}} +  \sum_{i:\bar\rho_{k(i)}=0} |\delta_i|\\
\le & \sqrt{\sum_{i:\bar\rho_{k(i)}\ne 0} \frac{(\wh \delta_i-\delta_i)^2}{\bar \rho_{k(i)}}} \sqrt{\sum_i \bar\rho_{k(i)}} + \sum_{i:\bar\rho_{k(i)}=0} |\delta_i|
\end{align*},
where $k(i)$ is the bucket that contains word $i$. $\sqrt{\sum_{i:\bar\rho_{k(i)}\ne 0} \frac{(\wh \delta_i-\delta_i)^2}{\bar \rho_{k(i)}}}$ is equal to the $\ell_2$ norm of vector $D^{-1} \delta-\wh \delta'$. Each $\bar \rho_{k(i)}$ is at most $e$ times the empirical marginal $\wh{\rho_i}$ and hence $\sum_i \bar \rho_{k(i)}$ is at most $e$, which gives us a $e$ upper bound for $\sqrt{\sum_i \bar\rho_{k(i)}}$. In terms of $\sum_{i:\bar\rho_{k(i)}=0} |\delta_i|$, reader can refer to the proof of Theorem~\ref{thm:main_tech} for the bound that $
\sum_{i:\bar\rho_{k(i)}=0} |\delta_i|\le 22e^{-\frac{k_0+1}{2}} + \frac{Me^{k_0}}{N} + e^{-e^{k_0-2}}.
$
Put everything together:
$$
\|\wh \delta-\delta\|_{\ell_1} = C((\frac{\log^2(N) }{\sqrt{w_{min}N}})^\frac{1}{2} + (\frac{M}{N})^\frac{1}{4}+e^{\frac{1}{4}(-e^{k_0-2}-k_0)}+(\frac{1}{e^{k_0} w^2_{min}})^{1/8}) + 22e^{-\frac{k_0+1}{2}} + \frac{Me^{k_0}}{N} + e^{-e^{k_0-2}}.
$$
Let $k_0\le 8\log(\frac{C\sqrt{R}}{\epsilon \sqrt{w_{min}}})+32$ and given $N=O(\frac{M}{w_{min}^4\eps^9})$ the right hand side of the above is bounded by $\eps$, as desired.
\end{proof}

\end{document}